\RequirePackage{fix-cm}

\documentclass[smallextended]{svjour3}       %
\smartqed  

\usepackage{url}
\usepackage{amsmath,amssymb}
\usepackage{natbib}
\usepackage{xcolor}
\usepackage{graphicx}
\usepackage[utf8]{inputenc}
\usepackage[colorinlistoftodos]{todonotes}

\DeclareUnicodeCharacter{2212}{-}
\begin{document}

\newcommand{\bm}[1]{\mathbf{#1}}
\newcommand{\trsize}{n}
\newcommand{\drsize}{m}
\newcommand{\tasize}{q}
\newcommand{\testdrsize}{{\widehat{m}}}
\newcommand{\testtasize}{{\widehat{q}}}
\newcommand{\transpose}{^\textnormal{T}}
\newcommand{\mrows}{a}
\newcommand{\mcols}{b}
\newcommand{\nrows}{c}
\newcommand{\ncols}{d}
\newcommand{\ulen}{f}
\newcommand{\vlen}{e}
\newcommand{\kf}{k}
\newcommand{\dtkf}{\kf_{\mathcal{D},\mathcal{T}}}
\newcommand{\ddkf}{\kf_{\mathcal{D},\mathcal{D}}}
\newcommand{\dkf}{\kf_{\mathcal{D}}}
\newcommand{\tkf}{\kf_{\mathcal{T}}}

\newcommand{\trdrinds}{d}
\newcommand{\trtainds}{t}
\newcommand{\tedrinds}{\widehat{d}}
\newcommand{\tetainds}{\widehat{t}}

\newcommand{\dkm}{\bm{D}}
\newcommand{\tkm}{\bm{T}}
\newcommand{\drug}{d}
\newcommand{\targ}{t}
\newcommand{\sdrug}{d'}
\newcommand{\trdrugs}{\mathcal{D}_{\textnormal{train}}}
\newcommand{\trtargets}{\mathcal{T}_{\textnormal{train}}}
\newcommand{\predfun}{f}

\newcommand{\testsetsize}{{\overline{\trsize}}}



\title{Generalized vec trick for fast learning of pairwise kernel models}






\author{Markus Viljanen         \and
        Antti Airola  \and
        Tapio Pahikkala
}

\institute{M. Viljanen \at
              Department of Computing, University of Turku, Finland\\
              \email{majuvi@utu.fi}           
           \and
           A. Airola \at
           Department of Computing, University of Turku, Finland\\
            \email{ajairo@utu.fi}           
           \and
           T. Pahikkala \at
           Department of Computing, University of Turku, Finland\\
            \email{aatapa@utu.fi}      
}

\date{Received: date / Accepted: date}

\maketitle

\begin{abstract}
Pairwise learning corresponds to the supervised learning setting where the goal is to make predictions for pairs of objects. Prominent applications include predicting drug-target or protein-protein interactions, or customer-product preferences.
In this work, we present a comprehensive review of pairwise kernels, that have been proposed for incorporating prior knowledge about the relationship between the objects.
Specifically, we consider the standard, symmetric and anti-symmetric Kronecker product kernels, metric-learning, Cartesian, ranking, as well as linear, polynomial and Gaussian kernels. Recently, a $O(\trsize\drsize+\trsize\tasize)$ time generalized vec trick algorithm, where $\trsize$, $\drsize$, and $\tasize$ denote the number of pairs, drugs and targets, was introduced for training kernel methods with the Kronecker product kernel. This was a significant improvement over previous $O(\trsize^2)$ training methods, since in most real-world applications $\drsize,\tasize << \trsize$.
In this work we show how all the reviewed kernels can be expressed as sums of Kronecker products, allowing the use of generalized vec trick for speeding up their computation. In the experiments, we demonstrate how the introduced approach allows scaling pairwise kernels to much larger data sets than previously feasible, and provide an extensive comparison of the kernels on a number of biological interaction prediction tasks.


\end{abstract}

\section{Introduction}

The goal of supervised learning is to learn an unknown function $f:\mathcal{X} \rightarrow \mathbb{R}$ from a set of training examples $Z=\{(x_i,y_i)\}_{i=1}^\trsize$ each consisting of an input $x_i\in\mathcal{X}$ and an associated label $y_i\in \mathbb{R}$. The learning algorithm returns a function that approximates the true function on the training set with the aim of generalizing to data unseen during the training phase.

In pairwise learning, each input $x$ is viewed as a pair of objects $x=(\drug,\targ)$ that we call here drugs $\drug\in\mathcal{D}$ and targets $\targ\in\mathcal{T}$. The task may, for example, then be to predict drug and target interaction $y=\predfun(\drug,\targ)$ values to test for novel interactions in drug discovery. This view is not unique and the inputs may be considered as paired in many different applications. For example, recommender system literature deals with ratings given to customer and product pairs \citep{basilico2004unifying,menon2010loglinear,rendle2010factorization}. Information retrieval can be formulated as predicting the relevance of query and document pairs \citep{Liu2011letorir}. Bioinformatics has utilized machine learning for protein-protein  \citep{Benhur2005,ruan2018improving},  protein-RNA  \citep{bellucci2011predicting}  and  drug-target  \citep{gonen2012predicting,pahikkala2015dti,cichonska2017computational,cichonska2018learning} interaction prediction.  Other applications include image labeling \citep{bernardino2015emabarrassing}, and link prediction in social networks \citep{pieter2005link}
Various terminology and frameworks have been used to describe the general learning problem (see e.g. \citet{waegeman2019multi} for overview). These include pairwise (kernel) learning  \citep{Benhur2005,park2009pairwise,cichonska2017computational,cichonska2018learning}, dyadic prediction \citep{menon2010loglinear,pahikkala2014twostep,Schafer2015}, pair-input prediction \citep{park2012flaws}, graph inference \citep{vert2007new}, link prediction \citep{pieter2005link,Kashima2009linkprob}, relational learning \citep{pahikkala2010reciprocalkm,waegeman2012learninggraded,pahikkala2013efficient}, multi-task \citep{bonilla2007taskspecific,bernard2017kernel} and as a special case zero-shot \citep{bernardino2015emabarrassing} learning.

Different fields often consider different but related pairwise prediction tasks. These tasks can be divided into settings where different methods are applicable and which have varying degrees of difficulty. For example, in recommender systems one often assumes that all customers and products belong to the training set and that there are some example interactions for each customer and each product \citep{basilico2004unifying}. Predictions are needed for (customer, product)-pairs where the rating is missing. In this setting, methods based on factorizing the interaction matrix can be used, and no explicit features are required.
However, in cold-start problems the task is to predict an interaction of a new customer and product pair, where we do not have any examples with the same customer or product in the training set. Basic factorization methods do not generalize to such settings, rather methods that make use of customer and product features are needed (sometimes called side information). In this work, we restrict our considerations to methods that can generalize to novel drugs and targets, rather than just imputing missing interactions between known ones.

Kernel methods are a standard method in supervised learning. They provide feature based generalization beyond training drugs and targets, and are used as a competitive method especially in the cold start setting. Kernel methods can be applied when the training data either have explicit feature vectors or implicit feature vectors are defined via positive semidefinite kernel functions. 
When drugs and targets have separate features or kernels, we can use pairwise kernels to define a kernel for the pair. One simple way to define a pairwise kernel, is to concatenate a feature vector for the drug and the target together, and apply a standard kernel such as polynomial or Gaussian on this feature vector. However, a large variety of different kernels specifically defined for pairwise data have been introduced in previous literature, starting with the introduction of the standard \citep{Benhur2005,basilico2004unifying,oyama2004} and symmetric \citep{Benhur2005} Kronecker product kernels.

 In this work we present a comprehensive review on pairwise kernels, and establish a joint framework under which the most commonly used of them can be expressed as linear combinations of Kronecker products. In particular, we cover the following kernels:
\begin{itemize}
\setlength\itemsep{0em}
    \item Linear kernel
    \item 2nd degreee polynomial kernel
    \item Gaussian kernel
    \item Kronecker product kernel \citep{Benhur2005,basilico2004unifying,oyama2004}
    \item symmetric Kronecker product kernel \citep{Benhur2005}
    \item anti-symmetric Kronecker product kernel \citep{pahikkala2010reciprocalkm} \item Cartesian kernel \citep{kashima2009pairwise}
    \item metric learning pairwise kernel \citep{vert2007new}
    \item ranking kernel  \citep{herbrich2000large,waegeman2012learninggraded}
\end{itemize}
\noindent In our framework, we represent the linear combinations corresponding to these kernels with specifically designed operator notation. The notation is not only mathematically elegant but, as we show below, enables the analysis of the kernel properties and assumptions, fast computation and easy implementation.

We start by introducing the two most fundamental assumptions. The first, what we call \textbf{pairwise data assumption}, is that both the drugs and targets  tend to appear several times as parts of different inputs in an observed data set. For example, the same drug $\drug_i$ may belong to two different examples $(\drug_i,\targ_j)$ and $(\drug_i,\targ_k)$. In particular, if $\trsize$ is the number of observed data and $\drsize$ and $\tasize$ denote the numbers of unique drugs and targets in the data, then $\trsize>>\drsize+\tasize$. This observation can be used to develop methods with computational shortcuts tailored specifically for the pairwise learning task, as we will elaborate below in more detail.

The second, what we refer to as \textbf{non-linearity assumption}, states another property inherent in pairwise learning problems, which is that the functions to be learned are usually not linear combinations of functions depending only of $\drug$ or $\targ$. The opposite case, where the function can be expressed as $\predfun(\drug,\targ)=\predfun_\drug(\drug)+\predfun_\targ(\targ)$ for some $\predfun_\drug$ and $\predfun_\targ$, would indicate that $\predfun(\drug_1,\targ_1 )>\predfun(\drug_2,\targ_1 )\Longrightarrow \predfun(\drug_1,\targ_2 )>\predfun(\drug_2,\targ_2 )$ for all drugs and targets, that is, if drug $\drug_1$ is more effective than drug $\drug_2$ against target $\targ_1$, then drug $\drug_1$ is also more effective than drug $\drug_2$ against target $\targ_2$. In other words, there would always be a single drug that would be the best choice for all targets (and vice versa). This is illustrated in Figure~\ref{chess} with a simple example, where the effect of drug on target is a function of the row and column number parities of both drugs and targets.
The 'chessboard' is a true pairwise data set where an interaction exists between even drugs and odd targets, or vice versa, corresponding to a XOR function between their parities, that is unlearnable with linear methods according to the classical result by \citet{minsky1969perceptrons}. In contrast, the 'tablecloth' a linear function of interaction strengths of odd drugs and odd targets, which are therefore completely independent of each other.

\begin{figure}[h]
    \centering
    \includegraphics[width=0.9\textwidth]{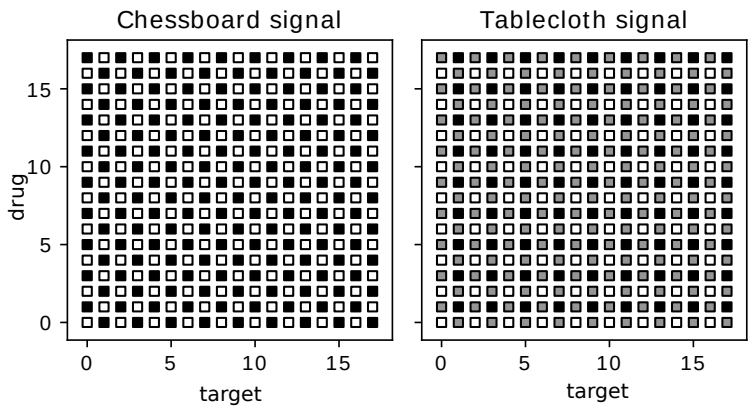}
    \caption{Illustration of pairwise data. The 'chessboard' is a XOR-function of drug and target parities, whereas 'tablecloth' is a SUM-function of the parities.}
    \label{chess}
\end{figure}

The runtime and in many cases the memory use of kernel solvers grow at least quadratically with respect to the number of pairs, and hence the use of pairwise kernels becomes infeasible in cases where the number of pairs is large. Faster training algorithms that avoid the costly step of building the pairwise kernel matrix have been previously proposed for certain specific cases. A fast solution to compute closed form solution is known for Kronecker product kernel, when minimizing ridge regression loss on so-called complete data that includes labels between all training drugs and targets \citep{bernardino2015emabarrassing,pahikkala2014twostep,pahikkala2013efficient,Stock2018twostep,Stock2018Algebraic}.
\citet{Kashima2009linkprob} show how computations for Cartesian kernel can be made faster, and computational shortcuts for speeding up the use of the ranking kernel are known for the ridge regression \citep{pahikkala2009efficient} and support vector machine algorithms \citep{kuo2014large}. Yet thus far there has been no unified approach that would allow to plug in any of the commonly used pairwise kernels to a kernel method training algorithm that guarantees better than $O(\trsize^2)$ scaling.

An exact computationally efficient algorithm has recently been proposed \citep{Airola2017gvt} for the special case of Kronecker product kernels when the data is not complete. The computational complexity of multiplying a vector with the kernel matrix is reduced from $O(\trsize^2)$ to $O(\trsize\drsize+\trsize\tasize)$. This improvement has already shown to have major practical relevance, for example, the winning team of a recently held IDG-DREAM
Drug-Kinase Binding Prediction Challenge on developing
models for predicting 
unexplored drug-target potencies, used this algorithm \citep{Cichonska2019Crowdsourced}.

In this work, we extend this result to present the first general $O(\trsize\drsize+\trsize\tasize)$ approach that simultaneously covers all the widely used pairwise kernels. This is a major improvement if the pairwise assumption holds, and the approach also allows accelerating the computation of the more traditional standard kernels, such as Gaussian and polynomial, if the data can be decomposed as pairwise. This is made possible by the proposed operator framework.

Finally, we perform an experimental comparison of different pairwise kernels on four different biological data sets, in which we compare the prediction performance, training time, number of training iterations and memory usage.
The kernels are compared with each other in the following four different prediction tasks: First, prediction of interaction strength between a drug and a target of which both have been observed in the training data as part of some other drug-target pair with a known interaction strength. Second, interaction strength prediction with novel targets that have not been observed in the training data as part of any known drug-target pair. Third, prediction with novel drugs, and fourth, prediction with both novel drugs and targets. As is shown in by the results, the prediction performances in these four tasks are tremendously different, hence underlying the importance that they should be separately considered in pairwise learning studies. Further, the results indicate that it is not at all self-evident that the expected prediction performance improvements of the nonlinear pairwise kernels over the linear ones implied by the nonlinearity assumption would always translate to practice.

To conclude, the major contributions of this work are as follows:
\begin{itemize}
    \item Review of the standard kernels for pairwise data that establishes a common operator based framework for analysing and implementing the kernels.
    \item The framework allows accelerating the computation of matrix-vector products with the pairwise kernels to $O(\trsize\drsize+\trsize\tasize)$ time, leading to considerably faster training methods.
    \item Comprehensive experimental comparison of the pairwise kernels on biological interaction data sets with four different prediction problems.
\end{itemize}

\section{Pairwise learning problem}

Given the spaces of drugs $\mathcal{D}$ and targets $\mathcal{T}$, the possible drug and target pairs are the Cartesian product $\mathcal{X}=\mathcal{D}\times\mathcal{T}$. The label space is denoted $\mathcal{Y}$, where $\mathcal{Y}=\mathbb{R}$ for regression and $\mathcal{Y}=\{0,1\}$ for classification. We further denote the joint space of the pairwise inputs and labels as $\mathcal{Z}=\mathcal{X}\times\mathcal{Y}$. The observed data set consists of $\trsize$ labeled drug-target pairs $Z_{\textnormal{obs}}=(X_{\textnormal{obs}},\bm{y})\in(\mathcal{D}\times\mathcal{T}\times\mathcal{Y})^{\trsize}$. 
We further define $\bm{\trdrinds}\in\mathcal{D}^\trsize$ and $\bm{\trtainds}\in\mathcal{T}^\trsize$ to be drug and target sequences such that $(\trdrinds_i,\trtainds_i)=x_i$.
Finally,
we let $\mathcal{D}_{\textnormal{obs}}$ 
and $\mathcal{T}_{\textnormal{obs}}$ denote the sets of drugs and targets observed in the sample and $\mathcal{Z}_{\textnormal{obs}}$ to denote the set of observed unique drug-target pairs, so that we have $\drsize=\arrowvert\mathcal{D}_{\textnormal{obs}}\arrowvert$ unique drugs and $\tasize=\arrowvert\mathcal{T}_{\textnormal{obs}}\arrowvert$ unique targets, 


Our goal is to learn a prediction function $\predfun:\mathcal{D}\times\mathcal{T}\rightarrow\mathcal{Y}$ from the training set, such that $\predfun$ can correctly predict the labels for a new pair $(\drug,\targ)\in\mathcal{D}\times\mathcal{T}$. The drug $\drug$ and target $\targ$ in the new pair may or may not belong to drugs $\mathcal{D}_{\textnormal{obs}}$  and targets $\mathcal{T}_{\textnormal{obs}}$ observed during training time.
Here, four different settings emerge, as illustrated in Figure~\ref{drawing-2}:
\begin{enumerate}
\item $\drug\in\mathcal{D}_{\textnormal{obs}}$  and $\targ\in\mathcal{T}_{\textnormal{obs}}$: prediction for known drugs and targets
\item $\drug\in\mathcal{D}_{\textnormal{obs}}$  and $\targ\notin\mathcal{T}_{\textnormal{obs}}$: prediction for novel targets
\item $\drug\notin\mathcal{D}_{\textnormal{obs}}$  and $\targ\in\mathcal{T}_{\textnormal{obs}}$: prediction for novel drugs
\item $\drug\notin\mathcal{D}_{\textnormal{obs}}$  and $\targ\notin\mathcal{T}_{\textnormal{obs}}$: prediction for novel drugs and targets
\end{enumerate}

\begin{figure}[h]
    \centering
    \includegraphics[width=0.6\textwidth]{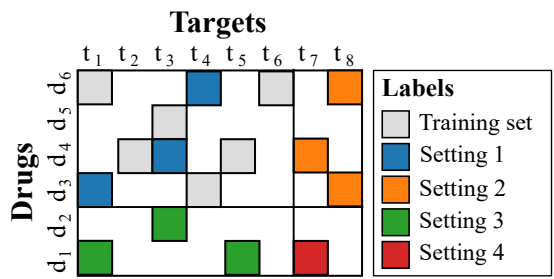}
    \caption{Illustration of a pairwise data set with (drug,target)-pairs and sparse labels. Different types of test sets corresponding to different settings are illustrated with different colors.}
    \label{drawing-2}
\end{figure}

In the literature specific settings in Figure~\ref{drawing-2} have been sometimes considered separately. For example, Setting 1 can be solved even without drug or target features using matrix factorization methods \citep{basilico2004unifying}. However, the latent representations learned by matrix factorization methods do not generalize to drugs and targets outside the training set (Settings
~2-4). The pairwise kernel learning approach considered in this work is applicable in all of the four settings.

\begin{table}[h!]
\renewcommand*{\arraystretch}{1.5}
\centering
\begin{tabular}{
 |p{0.12\columnwidth}|p{0.11\columnwidth}|p{0.65\columnwidth}| 
}
 \hline
 Setting & Task & Validation method \\ 
  \hline
 Setting 1 & \parbox[t]{\textwidth}{$\drug\in\mathcal{D}_{\textnormal{obs}}$\\ and\\ $\targ\in\mathcal{T}_{\textnormal{obs}}$} & Split samples $\mathcal{Z}_{\textnormal{obs}}=\mathcal{Z}_{\textnormal{train}}\cup\mathcal{Z}_{\textnormal{test}}$ \\
 
 Setting 2 & \parbox[t]{\textwidth}{$\drug\in\mathcal{D}_{\textnormal{obs}}$\\ and\\ $\targ\notin\mathcal{T}_{\textnormal{obs}}$} & 
 \parbox[t]{\textwidth}{Split targets $\mathcal{T}_{\textnormal{obs}}=\mathcal{T}_{\textnormal{train}}\cup\mathcal{T}_{\textnormal{test}}$.\\
 Then \\
 $
  (x_i,y_i)\in 
  \begin{cases}
    \mathcal{Z}_{\textnormal{train}}, & \textnormal{if } \trtainds_i\in\mathcal{T}_{\textnormal{train}}\\
    \mathcal{Z}_{\textnormal{test}}, & \textnormal{if } \trtainds_i\in\mathcal{T}_{\textnormal{test}}
  \end{cases}
$}
\\

 Setting 3 & \parbox[t]{\textwidth}{$\drug\notin\mathcal{D}_{\textnormal{obs}}$\\ and\\ $\targ\in\mathcal{T}_{\textnormal{obs}}$} & \parbox[t]{\textwidth}{Split drugs
 $\mathcal{D}_{\textnormal{obs}}=\mathcal{D}_{\textnormal{train}}\cup\mathcal{D}_{\textnormal{test}}$ \\
  Then \\
 $
  (x_i,y_i)\in 
  \begin{cases}
    \mathcal{Z}_{\textnormal{train}}, & \textnormal{if } \trdrinds_i\in\mathcal{D}_{\textnormal{train}}\\
    \mathcal{Z}_{\textnormal{test}}, & \textnormal{if } \trdrinds_i\in\mathcal{D}_{\textnormal{test}}
  \end{cases}
$}
\\
 Setting 4 & \parbox[t]{\textwidth}{$\drug\notin\mathcal{D}_{\textnormal{obs}}$\\ and\\ $\targ\notin\mathcal{T}_{\textnormal{obs}}$} & \parbox[t]{\textwidth}{Split both targets
 $\mathcal{T}_{\textnormal{obs}}=\mathcal{T}_{\textnormal{train}}\cup\mathcal{T}_{\textnormal{test}}$\\ and drugs
 $\mathcal{D}_{\textnormal{obs}}=\mathcal{D}_{\textnormal{train}}\cup\mathcal{D}_{\textnormal{test}}$  \\
  Then \\
 $
  (x_i,y_i)\in 
  \begin{cases}
    \mathcal{Z}_{\textnormal{train}}, & \textnormal{if } \trtainds_i\in\mathcal{T}_{\textnormal{train}} \textnormal{ and } \trdrinds_i\in\mathcal{D}_{\textnormal{train}}\\
    \mathcal{Z}_{\textnormal{test}}, & \textnormal{if } \trtainds_i\in\mathcal{T}_{\textnormal{test}} \textnormal{ and } \trdrinds_i\in\mathcal{D}_{\textnormal{test}}\\
    \mathcal{Z}_{\textnormal{ignored}}, & \textnormal{otherwise}
  \end{cases}
$}
\\
 
 \hline
\end{tabular}
\caption{Training and test set split in different settings}
\label{table:3}
\end{table}

Recent studies have highlighted, that the prediction performance and optimal choice of kernel and hyperparameters for a pairwise learning method crucially depend on the assumption of how the test pairs overlap with training data \citep{park2012flaws,pahikkala2015dti,Stock2018Algebraic}. An experimental observation made over a large variety of different studies was that Setting~1 is usually the easiest to predict accurately, followed by Settings~2 and 3, whereas making accurate predictions in Setting~4 tends to be very challenging.
As recommended in previous studies \citep{park2012flaws,pahikkala2015dti,Stock2018Algebraic}, we will always generate separate test sets for each of the four settings in the experiments to give a comprehensive view of how the learned prediction functions generalize to different types of test pairs. Depending on the amount of data, this can be implemented either with a single split to training and test sets, or by using cross-validation with repeated splits. 
The way the data splitting is implemented is defined in~Table
~\ref{table:3}.

\section{Learning algorithm}

In this section we present a supervised machine learning approach for learning with pairwise kernels. The computational shortcuts presented in this paper can be used to speed up any optimization approach whose computational complexity is dominated by multiplications of a pairwise kernel matrix with a vector, such as the truncated Newton method \citep{Airola2017gvt}. In this paper we focus on kernel ridge regression, as it is a widely used method that admits a closed form solution and simplifies the following considerations.

 To learn a prediction function, we consider the regularized empirical risk minimization problem 
\[
\predfun=\textnormal{argmin}_{\predfun\in\mathcal{H}} \left\{L(\bm{p},\bm{y}) + \dfrac{\lambda}{2} \|\predfun\|_{\mathcal{H}}^2\right\}
\]
where $\bm{p}\in\mathbb{R}^\trsize$  are the predicted outputs and $\bm{y}\in\mathbb{R}^\trsize$ the correct outputs, $L$ a convex nonnegative loss function and $\lambda > 0$ a regularization parameter. 

To define a kernel learning problem, let $\dtkf :(\mathcal{D}\times\mathcal{T})\times(\mathcal{D}\times\mathcal{T})\rightarrow\mathbb{R}$ be a positive semidefinite pairwise kernel function. Denote the kernel matrix containing the kernel evaluations between the drug-target pairs used to train the model as $\bm{K}\in\mathbb{R}^{\trsize\times\trsize}$ such that $\bm{K}_{i,j}=\dtkf((\trdrinds_i, \trtainds_i), (\trdrinds_j, \trtainds_j))$.
Choosing the reproducing kernel Hilbert space (RKHS)  associated with $\dtkf$ as the hypothesis space $\mathcal{H}$ for risk minimization, the representer theorem \citep{scholkopf2001generalized} implies that the minimizing function can be written as:
\[
\predfun(\drug,\targ)=\sum_{i=1}^{\trsize}a_i \dtkf((\trdrinds_i, \trtainds_i), (\drug,\targ))
\]
where $\bm{a}\in\mathbb{R}^\trsize$ is the vector of dual coefficients. Accordingly, the predictions for the training data can be written with the kernel matrix as $\bm{p} = \bm{K} \bm{a}$. 

Kernel ridge regression (see e.g. \citep{poggio2003mathematics}) is a special case of the regularized empirical risk minimization, where the loss function is the squared loss $L(\bm{p},\bm{y})={\|\bm{y}-\bm{p}\|}^2$. The optimization problem then has a direct solution in terms of matrix algebra. The ridge regression problem can be formulated as solving the dual parameter vector $\bm{a}\in\mathbb{R}^\trsize$:
\[
\bm{a}=\textnormal{argmin}_{\bm{a}\in\mathbb{R}^\trsize} {\|\bm{y}-\bm{K}\bm{a}\|}^2+\lambda\bm{a}^T \bm{K}\bm{a}
\]
It can be shown that this corresponds to solving the linear equation:
\begin{align}\label{linsystem}
(\bm{K}+\lambda \bm{I})\bm{a}=\bm{y}
\end{align}

Solving this system with a method that computes $\bm{K}$ requires at least $O(\trsize^2)$  time and memory, which is not practical in many pairwise learning problems, where $\trsize$ can be in the range of $10^5$ or more. A much more efficient solution can be found, when the kernel matrix can be expressed as a Kronecker product matrix. Assume we have a drug kernel function $\dkf :\mathcal{D}\times\mathcal{D}\rightarrow\mathbb{R}$ and a target kernel function $\tkf :\mathcal{T}\times\mathcal{T}\rightarrow\mathbb{R}$. The Kronecker product kernel is then defined as the product of the drug and target kernels $\dtkf((\drug,\targ),(\overline{\drug},\overline{\targ}))=\dkf(\drug,\overline{\drug})\tkf(\targ,\overline{\targ})$.

In the following considerations, we also use the following linear operator notation for the kernels. For the drug kernel, $\dkm\in\mathbb{R}
^{\mathcal{D}\times\mathcal{D}}$ such that $\dkm_{\drug,\overline{\drug}}=\dkf (\drug,\overline{\drug})$,
and for the target kernel $\tkm\in\mathbb{R}
^{\mathcal{T}\times\mathcal{T}}$ such that $\tkm_{\targ, \overline{\targ}}=\tkf (\targ,\overline{\targ})$.
For finite domains of drugs and targets, the operators can be considered as matrices, whose rows and columns are indexed with drugs and targets instead of positive integers. Their addition, scalar multiplication, transpose and Kronecker product of these operators also naturally extend to infinite and continuous domains. For example, the operator corresponding to Kronecker product kernel over drug and target kernels is 
$\dkm \otimes \tkm\in\mathbb{R}
^{(\mathcal{D}\times\mathcal{T})\times(\mathcal{D}\times\mathcal{T})}$ so that $(\dkm \otimes \tkm)_{(\drug,\targ),(\overline{\drug},\overline{\targ})}=\dkm_{\drug,\overline{\drug}} \tkm_{\targ,\overline{\targ}}$, and with the parenthesis notation we stress that both the rows and columns of the Kronecker product operator are indexed by drug-target pairs. Extending the matrix product is more involved in general but the products considered in this paper are always well-defined. This is enough for our purposes, and hence we avoid going into further technical details.

Let $\bm{R}(\bm{\trdrinds},\bm{\trtainds})\in\mathbb{R}^{\trsize\times(\mathcal{D}\times\mathcal{T})}$ denote the Kronecker product indexing operator, whose rows are indexed by a sample of $\trsize$ drug-target pairs and columns by all drug-target pairs in the space $\mathcal{D}\times\mathcal{T}$. Its values, as a function of the sequences $\bm{\trdrinds}\in\mathcal{D}^\trsize$ and $\bm{\trtainds}\in\mathcal{T}^\trsize$, are defined as follows:
\[
\bm{R}(\bm{\trdrinds},\bm{\trtainds})_{i,(\drug,\targ)} =
\begin{cases}
1 & \textnormal{if } (\drug,\targ)=(\drug_i,\targ_i)\\
0 & \textnormal{otherwise}
\end{cases}\;.
\]
Below, we omit explicitly writing $\bm{\trdrinds}$ and $\bm{\trtainds}$ for clarity when they are clear from the context. In the literature, this type of constructs are sometimes called sampling operators, as they select a finite sample from a space of possibilities.

For two samples of data, say $X=(\bm{\trdrinds},\bm{\trtainds})$ and $\overline{X}=(\overline{\bm{\trdrinds}},\overline{\bm{\trtainds}})$,
the kernel matrix containing all Kronecker product kernel evaluations between data in the first and second sample can then be expressed as $\bm{R}(\bm{\trdrinds},\bm{\trtainds}) (\dkm \otimes \tkm) \bm{R}(\overline{\bm{\trdrinds}},\overline{\bm{\trtainds}})^T$.
The second sample can be, for example, a validation sets used for selecting an appropriate value of the regularization parameter, the number of training iterations, or kernel parameter values. It can also be used for prediction performance evaluation of the final model with a separate test set or in general for performing predictions for data with unknown labels.

Substituting the kernel matrix of evaluations between the training data with itself to (\ref{linsystem}), we end up to the following linear system:
\begin{equation}\label{eq:1}
(\bm{R} (\dkm \otimes \tkm) \bm{R}^T + \lambda \bm{I} ) \bm{a} = \bm{y}
\end{equation}
This linear system can be solved iteratively, for example, with the minimal residual method \citep{saad1986gmres}, combined with early stopping. A single training iteration in Equation \ref{eq:1} requires matrix vector products of the form $\bm{u} \leftarrow (\bm{R} (\dkm \otimes \tkm) \bm{R}^T + \lambda \bm{I} ) \bm{u}$.
Given a vector of parameters $\bm{u}$, predictions for another sample of data not used in training
can be computed as a single matrix vector product $\bm{v} = \bm{R}(\overline{\bm{\trdrinds}},\overline{\bm{\trtainds}}) (\dkm \otimes \tkm) \bm{R}(\bm{\trdrinds},\bm{\trtainds})^T\bm{u}$, where  $\bm{u}\in\mathbb{R}^\trsize$  and $\bm{v}\in\mathbb{R}^\testsetsize$.

Table~\ref{table:1} presents the relevant dimensions associated to the matrix vector products. We next recollect the following result \citep{Airola2017gvt}  concerning matrix-vector products in which the matrix consists of a Kronecker product that is indexed from both left and right sides. This theorem is a generalization of Roth's column lemma \citep{roth1934columnlemma}, often known as the "vec-trick".

\begin{table}[h!]
	\centering
	\begin{tabular}{ |l|l| } 
		\hline
		$\trsize$ & the number of pairs in the first sample \\ 
		$\drsize$ & the number of unique drugs in the first sample \\ 
		$\tasize$ & the number of unique targets in the first sample \\ 
		$\overline{\trsize}$ & the number of pairs in the second sample\\ 
		$\overline{\drsize}$ & the number of unique drugs in the second sample\\ 
		$\overline{\tasize}$ & the number of unique targets  in the second sample\\ 
		\hline
	\end{tabular}
	\caption{Notation denoting the numbers of pairs, drugs and targets.}
	\label{table:1}
\end{table}

\begin{theorem}[\citet{Airola2017gvt}]\label{mainproposition}
Let 
\begin{align*}
\bm{R}(\bm{\trdrinds},\bm{\trtainds})&\in\mathbb{R}^{\trsize \times(\mathcal{D}\times\mathcal{T})}\\
\bm{R}(\overline{\bm{\trdrinds}},\overline{\bm{\trtainds}})&\in\mathbb{R}^{\overline{\trsize} \times(\mathcal{D}\times\mathcal{T})}\\
\bm{a}&\in\mathbb{R}^\trsize\\
\bm{p}&\in\mathbb{R}^{\overline{\trsize}}
\end{align*}
Then, the operation
\begin{align*}
    \bm{p}\leftarrow\bm{R}(\overline{\bm{\trdrinds}},\overline{\bm{\trtainds}})(\bm{D}\otimes \bm{T}) \bm{R}(\bm{\trdrinds},\bm{\trtainds})^T \bm{a}
\end{align*}
can be carried out in $O(\textnormal{min}(\overline{\tasize}\trsize+\drsize\overline{\trsize},\overline{\drsize}\trsize+\tasize\overline{\trsize}))$ time using a sparse Kronecker product multiplication algorithm known as the generalized vec-trick (GVT).
\end{theorem}

The theorem implies that in training, the Kronecker product kernel matrix can be multiplied with a dual parameter vector in $O(\tasize\trsize+\drsize\trsize)$ time. The cost of computing predictions simultaneously for a set of data not used for training is $O(\textnormal{min}(\overline{\tasize}\trsize+\drsize\overline{\trsize},\overline{\drsize}\trsize+\tasize\overline{\trsize}))$, where the overlined symbols denote the dimensions of the set for which the predictions are computed. This is much more efficient than the $O(\trsize^2)$ or  $O(\trsize\overline{\trsize})$ costs of explicitly forming the kernel matrices, since typically $\drsize,\tasize<<\trsize$ and $\overline{\drsize},\overline{\tasize}<<\overline{\trsize}$.



\section{Sum of Kronecker products framework for pairwise kernels}

\begin{table}[h!]
\centering
\begin{tabular}{ |p{2.5cm}|p{7.8cm}|} 

  \hline
  Kernel & $\dtkf((\drug,\targ),(\overline{\drug},\overline{\targ}))$ or $\ddkf((\drug,\drug'),(\overline{\drug},\overline{\drug'}))$ \\ 
  \hline
  Linear & $\dkf(\drug,\overline{\drug})+\tkf(\targ,\overline{\targ})$ \\ 
  Poly2D & ${(\dkf(\drug,\overline{\drug}) + \tkf(\targ,\overline{\targ}))}^2$\\
  Gaussian & $\exp(-\gamma\left\|\phi_{\mathcal{D}}(\drug)-\phi_{\mathcal{D}}(\overline{\drug})\right\|)
  \exp(-\gamma\left\|\phi_{\mathcal{T}}(\targ)-\phi_{\mathcal{T}}(\overline{\targ})\right\|)$
 \\
  Kronecker & $\dkf(\drug,\overline{\drug})\tkf(\targ,\overline{\targ})$ \\ 
  Symmetric & $\dkf(\drug,\overline{\drug})\dkf(\sdrug,\overline{\sdrug}) + \dkf(\drug,\overline{\sdrug})\dkf(\sdrug,\overline{\drug})$ \\ 
  Anti-Symmetric & $\dkf(\drug,\overline{\drug})\dkf(\sdrug,\overline{\sdrug}) - \dkf(\drug,\overline{\sdrug})\dkf(\sdrug,\overline{\drug})$  \\ 
 Ranking & $\dkf(\drug,\overline{\drug})-\dkf(\drug,\overline{\sdrug})-\dkf(\sdrug,\overline{\drug}) + \dkf(\sdrug,\overline{\sdrug})$ \\ 
  MLPK & $(\dkf(\drug,\overline{\drug})-\dkf(\drug,\overline{\sdrug})-\dkf(\sdrug,\overline{\drug}) + \dkf(\sdrug,\overline{\sdrug}))^2$  \\ 
 Cartesian & $\dkf(\drug,\overline{\drug})\delta(\targ=\overline{\targ})+ \delta(\drug=\overline{\drug})\tkf(\targ,\overline{\targ})$ \\ 
  
 \hline
\end{tabular}
\caption{Kernel functions of different pairwise kernels. We show that all of these kernels can be expressed as a combination Kronecker products of a separate drug and target kernel.}
\label{table:4b}
\end{table}

In this section we discuss different pairwise kernels presented in the literature and show how they can be expressed as sums of Kronecker products. Each matrix vector product can then be calculated as a sum of individual Knocker product terms. This allows the application of GVT shortcut to all of these kernels, which results in efficient algorithm for both training and making predictions. 

\begin{table}[h!]
\centering
\begin{tabular}{ 
|l|c|l|} 
 \hline
  Kernel &
  $\mathcal{D}\neq\mathcal{T} $ &
  $\phi_{\mathcal{D},\mathcal{T}}((\drug,\targ))$ or $\phi_{\mathcal{D},\mathcal{D}}((\drug,\drug'))$ \\
  \hline
  Linear & X & $(\phi_\mathcal{D}(\drug),\phi_\mathcal{T}(\targ))$ \\
  Poly2D & X & $(\phi_\mathcal{D}(\drug),\phi_\mathcal{T}(\targ)) \otimes (\phi_\mathcal{D}(\drug),\phi_\mathcal{T}(\targ))$ \\
  Kronecker  & X & $\phi_\mathcal{D}(\drug)\otimes\phi_\mathcal{T}(\targ)$ \\
  Symmetric  & & $\sqrt{1/2}(\phi_\mathcal{D}(\drug)\otimes\phi_\mathcal{D}(\sdrug) + \phi_\mathcal{D}(\sdrug)\otimes\phi_\mathcal{D}(\drug))$ \\
  Anti-Symmetric & & $\sqrt{1/2}(\phi_\mathcal{D}(\drug)\otimes\phi_\mathcal{D}(\sdrug) - \phi_\mathcal{D}(\sdrug)\otimes\phi_\mathcal{D}(\drug))$ \\
  Ranking & & $\phi_\mathcal{D}(\drug)-\phi_\mathcal{D}(\sdrug)$ \\
  MLPK & & $(\phi_\mathcal{D}(\drug)-\phi_\mathcal{D}(\sdrug))\otimes(\phi_\mathcal{D}(\drug)-\phi_\mathcal{D}(\sdrug))$ \\
  Cartesian & X & $(\phi_\mathcal{D}(\drug) \otimes e_\mathcal{T}, e_\mathcal{D} \otimes \phi_\mathcal{T}(\targ))$ \\

 \hline
\end{tabular}
\caption{Properties of different pairwise kernels. The middle column denotes whether the kernel allows heterogenous domains and the last column shows the feature map of the kernel.}
\label{table:4a}
\end{table}

Table \ref{table:4a} highlights an important limitation that applies to some of the kernels. These require homogeneous domains, i.e. they assume both objects in the pair belong to the same domain $\mathcal{D}=\mathcal{T}$, so that $x=(\drug,\targ)\in\mathcal{D}\times\mathcal{D}$. For the other kernels, we can have heterogeneous domains.
 Further, the Cartesian kernel is designed to be used in Setting 1 only, as it
 does not allow generalization to such drugs and targets that are not included in the training data.

The pairwise kernels can be motivated through feature mappings, since different pairwise kernel functions in Table \ref{table:4b} imply different implicit feature mappings for the pair as listed in Table \ref{table:4a}. 
The implicit drug and target feature mappings $\phi_\mathcal{D}:\mathcal{D}\rightarrow\mathbb{R}^r$  and $\phi_\mathcal{T}:\mathcal{T}\rightarrow\mathbb{R}^s$  are defined by the drug and target kernels $\dkf (\drug,\overline{\drug})=\langle \phi_\mathcal{D} (\drug),\phi_\mathcal{D} (\overline{\drug}) \rangle$  and  
$\tkf (\targ,\overline{\targ})=\langle \phi_\mathcal{T} (\targ),\phi_\mathcal{T} (\overline{\targ}) \rangle$. Then the feature vector of the pair is defined by the feature mapping $\phi_{\mathcal{D},\mathcal{T}}:\mathcal{D}\times\mathcal{T}\rightarrow\mathbb{R}^p$ corresponding to the pairwise kernel $\dtkf((\drug,\targ),(\overline{\drug},\overline{\targ}))=\langle \phi_{\mathcal{D},\mathcal{T}} (\drug,\targ),\phi_{\mathcal{D},\mathcal{T}} (\overline{\drug},\overline{\targ}) \rangle$. The claimed feature maps can be proven simply by computing the inner product and checking that it matches the definition of the kernel function. In the following, we discuss the implied pairwise feature vector $\phi_{\mathcal{D},\mathcal{T}}((\drug,\targ)):=(x_1^{\drug,\targ},...,x_p^{\drug,\targ})\in\mathbb{R}^p$ of each pairwise kernel in terms of the drug $\phi_{\drug}(\drug):=(x_1^\drug,...,x_r^\drug)\in\mathbb{R}^r$ and the target $\phi_{\targ}(\targ):=(x_1^\targ,...,x_s^\targ)\in\mathbb{R}^s$ feature vectors. This motivates the kernels and demonstrates the intuition behind using a specific kernel for a specific task.

\subsection{Linear}
The pairwise linear kernel is computed as the linear kernel on the concatenated feature vector. The feature vector is the concatenation of the drug and target feature vectors $\bm{x}^{\drug,\targ}=(\bm{x}^{\drug},\bm{x}^{\targ})$. The resulting features consists of the union of original drug features $(x_i^\drug)_{i=1...r}$  and target features $(x_i^\targ)_{i=1...s}$. In this feature mapping, each feature contributes equally to interaction strength in every drug and target pair. Interaction is predicted simply by the presence or absence of certain features in the drug or the target, regardless of which drug and target pair is being tested. Given drug $d$ and target $t$, the predicted interaction of the drug on the target is given by $\predfun(\drug,\targ)=\langle \bm{w}^\drug, \bm{x}^\drug \rangle + \langle \bm{w}^\targ, \bm{x}^\targ \rangle$.  This implies a global ordering of drugs, where drugs and targets are completely decoupled. If drug $\drug_1$ is more effective than drug $\drug_2$ against target $\targ_1$, then drug $\drug_1$ is also more effective than drug $\drug_2$ against target $\targ_2$: $\predfun(\drug_1,\targ_1 )>\predfun(\drug_2,\targ_1 )\Longrightarrow(\drug_1,\targ_2 )>\predfun(\drug_2,\targ_2 )$. In the resulting model, some drugs and targets simply have more interactions than others, but there are no interactions between drug and target features. The artificial chessboard problem illustrated in Figure~\ref{chess} is an example of a data set, that is impossible to model using the pairwise linear kernel. 

\subsection{Polynomial}
The pairwise polynomial kernel is computed as the polynomial kernel on the concatenated feature vector.  On a second degree polynomial kernel without bias, the feature vector is the tensor product of the concatenated feature vector with itself $\bm{x}^{\drug,\targ}=(\bm{x}^{d},\bm{x}^{t})\otimes(\bm{x}^{d},\bm{x}^{t})$. The resulting features include three types of terms: self interactions between drug features $(x_i^d x_j^d)_{i=1...r,j=1...r}$, pairwise interactions between drug and target features $(x_i^d x_j^t )_{i=1...r,j=1...s}$, and self interactions between target features $(x_i^t x_j^t)_{i=1...s,j=1...s}$. The self interactions contribute to a global ordering of drugs and targets, similar to the linear kernel. However, the pairwise interactions model actual interactions of drug and target features: a drug and target pair may be interacting if for example the features indicate that a certain chemical structure in a drug binds to a certain receptor on a target.

\subsection{Gaussian}

The pairwise Gaussian kernel is defined as the Gaussian kernel on the concatenated feature vector. This kernel $\exp(-\gamma\left\|(\bm{x}^{d},\bm{x}^{t})-(\overline{\bm{x}}^{d},\overline{\bm{x}}^{t})\right\|)=
\exp(-\gamma\left\|\bm{x}^{d}-\overline{\bm{x}}^{d}\right\|)
\exp(-\gamma\left\|\bm{x}^{t}-\overline{\bm{x}}^{t}\right\|)$ can be expressed as product of Gaussian drug and target kernels. This is a special case of the Kronecker product kernel, and will thus not be considered separately in the following.

\subsection{Kronecker product}
The Kronecker product kernel \citep{Benhur2005,basilico2004unifying,oyama2004} is computed as the product of drug and target kernels. The feature vector is given as a tensor product of the drug and target feature vectors $\bm{x}^{\drug,\targ}=\bm{x}^{d}\otimes\bm{x}^{t}$. The resulting feature vector consists of simply all the pairwise interactions $(x_i^d x_j^t )_{i=1...r,j=1...s}$. These are same as the pairwise interactions in the polynomial kernel with self-interations excluded. The Kronecker product kernel can be motivated as the simplest kernel that models actual pairwise interactions in drug and target features. The Kronecker kernel is an universal kernel, if the drug and target kernels are universal (e.g. Gaussian) \citep{waegeman2012learninggraded}.

\subsection{Symmetric and Anti-symmetric kernels}
If we assume homogeneous domains, feature vectors can be written as a sum of symmetric and anti-symmetric parts
$\phi_{\mathcal{D},\mathcal{D}}((\drug,\drug'))=1/2(\phi_{\mathcal{D},\mathcal{D}}((\drug,\drug'))+\phi_{\mathcal{D},\mathcal{D}}((\drug',\drug)))+1/2(\phi_{\mathcal{D},\mathcal{D}}((\drug,\drug'))-\phi_{\mathcal{D},\mathcal{D}}((\drug',\drug)))$.
The symmetric Kronecker kernel \citep{Benhur2005} is motivated by applying the symmetrization to the Kronecker kernel feature vector. This results in a tensor product of the drug and target feature vectors with only symmetric parts $\bm{x}^{\drug,\drug'}=1/2(\bm{x}^{d}\otimes\bm{x}^{d'}+\bm{x}^{d'}\otimes\bm{x}^{d})$. The resulting features consist of all symmetric pairwise interactions $(1/2(x_i^d x_j^{d'}+x_i^{d'} x_j^\drug))_{i=1...r,j=1...r}$. When all interactions are known to be symmetric by definition, the symmetric Kronecker kernel is sometimes referred to as the Kronecker kernel in the literature. Several works have analysed the theoretical properties of the symmetric and antisymmetric Kronecker kernels
 \citep{pahikkala2010reciprocalkm,waegeman2012learninggraded,brunner2012pairwise,pahikkala2015spectral,Gnecco2017symmetry,Gnecco2018symmetric}.

\subsection{Ranking}
The feature vector of the ranking kernel is the difference of drug and target feature vectors $\bm{x}^{\drug,\drug'}=\bm{x}^{d}-\bm{x}^{d'}$, which are assumed to belong to the same domain \citep{herbrich2000large,waegeman2012learninggraded}.  The resulting features consist of pairwise differences $(x_i^d - x_i^{d'} )_{i=1...r}$. The ranking kernel can model ranking representable relations, i.e. relations constructed from some utility function h such that $\predfun(\drug,\drug')=h(\drug) - h(\drug')$. For the ranking kernel  $\predfun(\drug,\drug')=\langle \bm{w}, \bm{x}^d \rangle - \langle \bm{w}, \bm{x}^{d'} \rangle$, which provides a global ranking of drugs based on their feature representation.
The ranking kernel can be considered as an anti-symmetric linear kernel as can be observed from the operator notation below.

\citet{pahikkala2009efficient} show that the ranking kernel matrix can be computed using the oriented incidence operator
$\bm{M}\in\mathbb{R}^{\mathcal{D}\times n}$ where
\[
\bm{M}_{\drug,(\drug_i,\drug'_i)} =
\begin{cases}
1 & \textnormal{if } \drug_i=\drug\\
-1 & \textnormal{if } \drug'_i=\drug\\
0 & \textnormal{otherwise}
\end{cases}\;.
\]
as $\bm{M}\transpose\bm{D}\bm{M}$. Since $\bm{M}$ can be implemented with a sparse matrix, this allows efficient kernel matrix vector multiplication in $O(m^2+n)$ time without need to use GVT. 

\subsection{MLPK}
The MLPK kernel \citep{vert2007new} is computed as the ranking kernel squared. The feature vector is given by the tensor product of the pairwise difference vector with itself $\bm{x}^{\drug,\drug'}=(\bm{x}^{d}-\bm{x}^{d'})\otimes(\bm{x}^{d}-\bm{x}^{d'})$. The features consists of all pairwise interactions of pairwise differences $((x_i^d - x_i^{d'})(x_j^d - x_j^{d'}))_{i=1...r,j=1...r}$. This models interaction of a pair in the terms of how similar the drug and the target in the pair are. The formulation compares both elementwise differences, and possible interactions between the differences. The MLPK kernel can also be motivated as a distance learning problem by adding an extra parameter constraint to the standard SVM optimization problem \citep{vert2007new}. There, the goal is to learn a linear map such that the function is modelled by the Euclidean distance metric between feature vectors: learn a positive semidefinite matrix $\bm{M}$ such that $\predfun(\drug,\drug')=(\bm{x}^{d}-\bm{x}^{d'})^T \bm{M} (\bm{x}^{d}-\bm{x}^{d'})$.

\subsection{Cartesian}
The Cartesian kernel \citep{kashima2009pairwise} is computed as the drug kernel when the targets match, and the target kernel when the drugs match. The feature vector is given as a concatenation of the drug feature vector (target specific) and the target feature vector (drug specific) $\bm{x}^{\drug,\targ}=(\bm{x}^{d}\otimes e_\targ, e_\drug\otimes\bm{x}^{t})$. The resulting features are sparse with nonzero terms $(x_i^d\delta(\targ=\targ_j))_{i=1...r,j=1...q}$ and $(\delta(\drug=\drug_i)x_j^t)_{i=1...m,j=1...s}$ corresponding to drug and target specific features. The full parameter vector $\bm{w}$ can be partitioned into drug specific $(\bm{w}^d)_{d\in\mathcal{D}}$  and target specific $(\bm{w}^t)_{t\in\mathcal{T}}$  parameters, with separate parameters learned for each drug and target. This means that target features may have different effects, depending on the drug, and vice versa. In this sense the learned model includes pairwise interactions, but it does not utilize information between similar interactions in different pairs and cannot generalize to drugs or targets that have not been seen in the training set. \citet{kashima2009pairwise} show that the Cartesian kernel can be represented as a Kronecker sum, and thus using the standard vec trick \citep{roth1934columnlemma} kernel matrix multiplication can be done in $O(m^2q+q^2m)$ time. In this work, we improve on this result.

\subsection{Efficient computation of pairwise kernels}

In this section we show how the pairwise kernel matrices of the above described kernels can be conveniently written as sums of Kronecker product matrices. For this purpose, we make the following definitions.
\begin{definition}[Commutation and unification operators]\label{commdef}
The commutation operator
$\bm{P}\in\mathbb{R}^{(\mathcal{T}\times\mathcal{D})\times(\mathcal{D}\times\mathcal{T})}$ has its values defined as
\[
\bm{P}_{(\targ,\drug),(\overline{\drug},\overline{\targ})}=\left\{
\begin{array}{ll}
1&\textnormal{if }\drug=\overline{\drug}\textnormal{ and }\targ=\overline{\targ}\\
0&\textnormal{otherwise}
\end{array}
\right.\;.
\]
Note that if the domains $\mathcal{D}$ and $\mathcal{T}$ are different, the row indexing of any operator will be changed from $\mathcal{D}\times\mathcal{T}$ to $\mathcal{T}\times\mathcal{D}$ if multiplied from left with $\bm{P}$. Its inverse operator $\bm{P}^{T}\in\mathbb{R}^{(\mathcal{D}\times\mathcal{T})\times(\mathcal{T}\times\mathcal{D})}$ is defined analogously, by switching the drug and target domains. The values are also defined similarly when $\mathcal{D}=\mathcal{T}$ but in this case we use the notation
\[
\bm{P}_{(\sdrug,\drug),(\overline{\drug},\overline{\sdrug})}=\left\{
\begin{array}{ll}
1&\textnormal{if }\drug=\overline{\drug}\textnormal{ and }\sdrug=\overline{\sdrug}\\
0&\textnormal{otherwise}
\end{array}
\right.\;.
\]

The unification operator $\bm{Q}\in\mathbb{R}^{(\mathcal{D}\times\mathcal{T})\times(\mathcal{D}\times\mathcal{D})}$ has its values defined as:
\[
\bm{Q}_{(\drug,\targ),(\overline{\drug},\overline{\sdrug})}=
\left\{\begin{array}{ll}
1&\textnormal{if }\drug=\overline{\drug}=\overline{\sdrug}\\
0&\textnormal{otherwise}
\end{array}\right.\;.
\]
The corresponding unification operator $\bm{Q}\in\mathbb{R}^{(\mathcal{T}\times\mathcal{D})\times(\mathcal{T}\times\mathcal{T})}$ is defined analogously, by switching the drug and target domains. The values are also defined similarly when $\mathcal{D}=\mathcal{T}$ but in this case we use the notation
\[
\bm{Q}_{(\drug,\sdrug),(\overline{\drug},\overline{\sdrug})}=
\left\{\begin{array}{ll}
1&\textnormal{if }\drug=\overline{\drug}=\overline{\sdrug}\\
0&\textnormal{otherwise}
\end{array}\right.\;.
\]
For convenience, we also give the values of the product of operators $\bm{PQ}\in\mathbb{R}^{(\mathcal{D}\times\mathcal{T})\times(\mathcal{T}\times\mathcal{T})}$:
\[
\bm{PQ}_{(\drug,\targ),(\overline{\targ},\overline{\targ'})}=
\left\{\begin{array}{ll}
1&\textnormal{if }\targ=\overline{\targ}=\overline{\targ'}\\
0&\textnormal{otherwise}
\end{array}\right.
\]
as this product is also heavily used in the forthcoming considerations.
\end{definition}

In the following example, we illustrate finite dimensional examples of both commutation and unification operators that are, due to their finiteness, representable as matrices.
\begin{example}
Consider a finite space of drugs of size $\arrowvert\mathcal{D}\arrowvert=3$ and a finite space of targets $\arrowvert\mathcal{T}\arrowvert=2$. Then, the commutation operator $\bm{P}\in\mathbb{R}^{(\mathcal{T}\times\mathcal{D})\times(\mathcal{D}\times\mathcal{T})}$ can be represented as the following matrix:
\begin{eqnarray*}
\bm{P}
=\left(
\begin{array}{ccccccc}
1 &0 &0 &0 &0 &0\\
0 &0 &1 &0 &0 &0\\
0 &0 &0 &0 &1 &0\\
0 &1 &0 &0 &0 &0\\
0 &0 &0 &1 &0 &0\\
0 &0 &0 &0 &0 &1\\
\end{array}
\right)\;,
\end{eqnarray*}
where rows and columns are arranged according to the natural order of the target-drug and drug-target pairs, respectively. This is in the literature known as the commutation matrix (see e.g. \citet{magnus1979commutation}). The unification operator $\bm{Q}\in\mathbb{R}^{(\mathcal{D}\times\mathcal{T})\times(\mathcal{D}\times\mathcal{D})}$ can be represented as the matrix
\begin{eqnarray*}
\bm{Q}
=\left(
\begin{array}{cccccccccc}
1 &0 &0 &0 &0 &0 &0 &0 &0\\
1 &0 &0 &0 &0 &0 &0 &0 &0\\
0 &0 &0 &0 &1 &0 &0 &0 &0\\
0 &0 &0 &0 &1 &0 &0 &0 &0\\
0 &0 &0 &0 &0 &0 &0 &0 &1\\
0 &0 &0 &0 &0 &0 &0 &0 &1\\
\end{array}
\right)\;,
\end{eqnarray*}
where the rows and columns are arranged in  the natural order of the drug-target pairs and drug-drug pairs, respectively.
\end{example}

From the above definition of the commutation and unification operators, we obtain a cheat sheet of rules indicated by the following lemma:
\begin{theorem}
For $\bm{P}\in\mathbb{R}^{(\mathcal{T}\times\mathcal{D})\times(\mathcal{D}\times\mathcal{T})}$, we have
\begin{align*}
\bm{P}\bm{P}^{T}&=\bm{P}^{T}\bm{P}=\bm{I}\\
\bm{P}[\bm{D}\otimes\bm{T}]&=[\bm{T}\otimes\bm{D}]\bm{P}
\\
\bm{P}[\bm{D}\otimes\bm{T}]\bm{P}^{T}&=[\bm{T}\otimes\bm{D}]
\end{align*}
And for $\bm{P}\in\mathbb{R}^{(\mathcal{D}\times\mathcal{D})\times(\mathcal{D}\times\mathcal{D})}$
\begin{align*}
\bm{P}&=\bm{P}^{T}
\\
\bm{P}[\bm{D}\otimes\bm{D}]&=[\bm{D}\otimes\bm{D}]\bm{P}
\\
\bm{P}[\bm{D}\otimes\bm{D}]\bm{P}&=[\bm{D}\otimes\bm{D}]
\end{align*}
Further, for $\bm{Q}\in\mathbb{R}^{(\mathcal{T}\times\mathcal{D})\times(\mathcal{D}\times\mathcal{D})}$, we have
\begin{align*}
\bm{Q}[\bm{D}\otimes\bm{D}]\bm{Q}^T&=[\bm{D}^{\odot 2}\otimes\bm{1}]\;.
\end{align*}
where $\bm{D}^{\odot 2}$ denotes the elementwise square of $\bm{D}$, and $\bm{1}\in\mathbb{R}^{\mathcal{T}\times\mathcal{T}}$ is an operator with all values equal to one.

For the values, we have
\begin{align*}
[\bm{Q}(\bm{D}\otimes \bm{D})\bm{Q}^T]_{(\drug,\targ),(\overline{\drug},\overline{\targ})}
&= (\bm{D}\otimes \bm{D})_{(\drug,\drug),(\overline{\drug},\overline{\drug})}\;,
\end{align*}
where $\bm{Q}\in\mathbb{R}^{(\mathcal{D}\times\mathcal{T})\times(\mathcal{D}\times\mathcal{D})}$,
and
\begin{align*}
[\bm{P}\bm{Q}(\bm{T}\otimes \bm{T})\bm{Q}^{T}\bm{P}^T]_{(\drug,\targ),(\overline{\drug},\overline{\targ})}
&= (\bm{T}\otimes \bm{T})_{(\targ,\targ),(\overline{\targ},\overline{\targ})}\;,
\end{align*}
where $\bm{Q}\in\mathbb{R}^{(\mathcal{D}\times\mathcal{T})\times(\mathcal{T}\times\mathcal{T})}$.

Finally, if $\mathcal{D}=\mathcal{T}$, we further have:
\begin{align*}
(\bm{D}\otimes \bm{D})_{(\drug,\sdrug),(\overline{\drug},\overline{\sdrug})}
&= (\bm{D}\otimes \bm{D})_{(\sdrug,\drug),(\overline{\sdrug},\overline{\drug})}
\\
(\bm{P}(\bm{D}\otimes \bm{D}))_{(\drug,\sdrug),(\overline{\drug},\overline{\sdrug})}
&= (\bm{D}\otimes \bm{D})_{(\sdrug,\drug),(\overline{\drug},\overline{\sdrug})}
\\
(\bm{Q}(\bm{D}\otimes \bm{D}))_{(\drug,\sdrug),(\overline{\drug},\overline{\sdrug})}
&= (\bm{D}\otimes \bm{D})_{(\drug,\drug),(\overline{\drug},\overline{\sdrug})}
\\
((\bm{D}\otimes \bm{D})\bm{Q}^{T})_{(\drug,\sdrug),(\overline{\drug},\overline{\sdrug})}
&= (\bm{D}\otimes \bm{D})_{(\drug,\sdrug),(\overline{\drug},\overline{\drug})}
\\
(\bm{P}\bm{Q}(\bm{D}\otimes \bm{D}))_{(\drug,\sdrug),(\overline{\drug},\overline{\sdrug})}
 &= (\bm{D}\otimes \bm{D})_{(\sdrug,\sdrug),(\overline{\drug},\overline{\sdrug})}
\\
((\bm{D}\otimes \bm{D})\bm{Q}^T\bm{P})_{(\drug,\sdrug),(\overline{\drug},\overline{\sdrug})}
&= (\bm{D}\otimes \bm{D})_{(\drug,\drug),(\overline{\sdrug},\overline{\sdrug})}
\\
(\bm{Q}(\bm{D}\otimes \bm{D})\bm{Q}^{T})_{(\drug,\sdrug),(\overline{\drug},\overline{\sdrug})}
&= (\bm{D}\otimes \bm{D})_{(\drug,\drug),(\overline{\drug},\overline{\drug})}
\\
(\bm{P}\bm{Q}(\bm{D}\otimes \bm{D})\bm{Q}^{T})_{(\drug,\sdrug),(\overline{\drug},\overline{\sdrug})}
&= (\bm{D}\otimes \bm{D})_{(\sdrug,\sdrug),(\overline{\drug},\overline{\drug})}
\\
(\bm{Q}(\bm{D}\otimes \bm{D})\bm{Q}^T\bm{P})_{(\drug,\sdrug),(\overline{\drug},\overline{\sdrug})}
&= (\bm{D}\otimes \bm{D})_{(\drug,\drug),(\overline{\sdrug},\overline{\sdrug})}
\\
(\bm{P}\bm{Q}(\bm{D}\otimes \bm{D})\bm{Q}^T\bm{P})_{(\drug,\sdrug),(\overline{\drug},\overline{\sdrug})}
&= (\bm{D}\otimes \bm{D})_{(\sdrug,\sdrug),(\overline{\sdrug},\overline{\sdrug})}
\end{align*}
\end{theorem}
\begin{proof}
The listed results are straightforward operator algebraic manipulations based on Definition~\ref{commdef}.\qed
\end{proof}
From the above results, we can conclude the following results concerning certain specific pairwise kernels in particular:
\begin{corollary}\label{pairwisetheorem}
The kernel matrices of the linear, second order polynomial, Kronecker product, Cartesian, symmetric, anti-symmetric, ranking and metric learning pairwise kernels for two samples of data, say $X=(\bm{\trdrinds},\bm{\trtainds})$ and $\overline{X}=(\overline{\bm{\trdrinds}},\overline{\bm{\trtainds)}}$, can be expressed as $\bm{R}(\overline{\bm{\trdrinds}},\overline{\bm{\trtainds}})\bm{K}_{\mathcal{D},\mathcal{T}}\bm{R}(\bm{\trdrinds},\bm{\trtainds})$, where $\bm{K}_{\mathcal{D},\mathcal{T}}$ is the corresponding operator of all kernel values as follows:

\begin{tabular}{ |p{2.5cm}|p{7.9cm}| } 
  \hline
  Kernel & $\bm{K}_{\mathcal{D},\mathcal{T}}\in\mathbb{R}^{(\mathcal{D}\times\mathcal{T})\times(\mathcal{D}\times\mathcal{T})}$ or $\bm{K}_{\mathcal{D},\mathcal{D}}\in\mathbb{R}^{(\mathcal{D}\times\mathcal{D})\times(\mathcal{D}\times\mathcal{D})}$ \\ 
  \hline
  Linear & $
  \dkm\otimes \bm{1} + \bm{1} \otimes \tkm$ \\ 
  Poly2D & ${\bm{Q}(\bm{D}\otimes \bm{D})\bm{Q}^T + 2 \bm{D} \otimes \bm{T} + \bm{P}\bm{Q}(\bm{T} \otimes \bm{T})\bm{Q}^T\bm{P}}$ \\ 
  Kronecker & $\dkm \otimes \tkm$ \\ 
  Cartesian & $\dkm \otimes \bm{I} + \bm{I} \otimes \tkm $ \\
  Symmetric & $ (\bm{P}+\bm{I}) (\dkm \otimes \dkm)$ \\ 
  Anti-Symmetric & $(\bm{P}-\bm{I}) (\dkm \otimes \dkm)$ \\ 
  Ranking & $(\bm{I}-\bm{P})(\bm{D} \otimes \bm{1})(\bm{I}-\bm{P})$\\ 
  MLPK & $(\bm{I}+\bm{P})(\bm{I}-\bm{Q})(\bm{D}\otimes \bm{D})(\bm{I}-\bm{Q})^T(\bm{I}+\bm{P})$\\
 \hline
\end{tabular}
Their products with vectors can be computed with GVT in $O(\textnormal{min}(\overline{\tasize}\trsize+\drsize\overline{\trsize},\overline{\drsize}\trsize+\tasize\overline{\trsize}))$ time.
\end{corollary}

\begin{proof}
We first show that the kernel matrices over the whole domain of $\mathcal{D}$ and $\mathcal{T}$ can be compactly expressed with the operator notation and show the indexed case afterwards.


\begin{equation*}
\begin{split}
\bm{K}^{\textnormal{Kronecker}}_{(\drug,\targ),(\overline{\drug},\overline{\targ})}
&=\dkf(\drug,\overline{\drug})\tkf(\targ,\overline{\targ})\\
&= \bm{D}_{\drug,\overline{\drug}}\bm{T}_{\targ,\overline{\targ}}\\
&= {(\bm{D}\otimes \bm{T})}_{(\drug,\targ),(\overline{\drug},\overline{\targ})} \\
\end{split}
\end{equation*}

\begin{equation*}
\begin{split}
\bm{K}_{(\drug,\targ),(\overline{\drug},\overline{\targ})}^{\textnormal{Linear}}
    &= \dkf(\drug,\overline{\drug}) + \tkf(\targ,\overline{\targ})\\
    &= {(\bm{D}\otimes \bm{1} + \bm{1} \otimes \bm{T})}_{(\drug,\targ),(\overline{\drug},\overline{\targ})} \\
\end{split}
\end{equation*}

\begin{equation*}
\begin{split}
  \bm{K}^{\textnormal{Poly2D}}_{(\drug,\targ),(\overline{\drug},\overline{\targ})}
  &= {(\dkf(\drug,\overline{\drug}) + \tkf(\targ,\overline{\targ}))}^2\\
  &= \dkf(\drug,\overline{\drug})\dkf(\drug,\overline{\drug}) + 2\dkf(\drug,\overline{\drug})\tkf(\targ,\overline{\targ})+\tkf(\targ,\overline{\targ})\tkf(\targ,\overline{\targ})\\
  &= (\bm{Q}(\bm{D}\otimes \bm{D})\bm{Q}^T + 2 \bm{D} \otimes \bm{T} + \bm{P}\bm{Q}(\bm{T} \otimes \bm{T})\bm{Q}^T\bm{P})_{(\drug,\targ),(\overline{\drug},\overline{\targ})}
\end{split}
\end{equation*}

\begin{equation*}
\begin{split}
  \bm{K}^{\textnormal{Cartesian}}_{(\drug,\targ),(\overline{\drug},\overline{\targ})}
  &= \dkf(\drug,\overline{\drug})\delta(\targ,\overline{\targ}) + \delta(\drug,\overline{\drug})\tkf(\targ,\overline{\targ})\\
  &= {(\bm{D}\otimes \bm{I} + \bm{I} \otimes \bm{T})}_{(\drug,\targ),(\overline{\drug},\overline{\targ})}
\end{split}
\end{equation*}



\begin{equation*}
\begin{split}
\bm{K}^{\textnormal{Symmetric}}_{(\drug,\sdrug),(\overline{\drug},\overline{\sdrug})}
  &= \dkf(\drug,\overline{\drug})\dkf(\sdrug,\overline{\sdrug}) + \dkf(\sdrug, \overline{\drug})\dkf(\drug,\overline{\sdrug})\\
&={(\bm{D}\otimes \bm{D})}_{(\drug,\sdrug),(\overline{\drug},\overline{\sdrug})}+{(\bm{D}\otimes \bm{D})}_{(\sdrug,\drug),(\overline{\drug},\overline{\sdrug})}\\
&= \left((\bm{P}+\bm{I})(\bm{D}\otimes \bm{D})\right)_{(\drug,\sdrug),(\overline{\drug},\overline{\sdrug})}
\end{split}
\end{equation*}

\begin{equation*}
\begin{split}
  \bm{K}_{(\drug,\sdrug),(\overline{\drug},\overline{\sdrug})}^{\textnormal{Anti-Symmetric}} 
  &= \dkf(\drug,\overline{\drug})\dkf(\sdrug,\overline{\sdrug}) - \dkf(\sdrug, \overline{\drug})\dkf(\drug,\overline{\sdrug})\\
&= \left((\bm{P}-\bm{I})(\bm{D}\otimes \bm{D})\right)_{(\drug,\sdrug),(\overline{\drug},\overline{\sdrug})}
\end{split}
\end{equation*}

\begin{equation*}
\begin{split}
  \bm{K}_{(\drug,\sdrug),(\overline{\drug},\overline{\sdrug})}^{\textnormal{Ranking}} 
  &= \dkf(\drug, \overline{\drug}) - \dkf(\sdrug,\overline{\drug}) - \dkf(\drug, \overline{\sdrug}) + \dkf(\sdrug, \overline{\sdrug})\\
  &= [(\bm{I}-\bm{P})(\bm{D} \otimes \bm{1})(\bm{I}-\bm{P})]_{(\drug,\sdrug),(\overline{\drug},\overline{\sdrug})}\\
\end{split}
\end{equation*}

\begin{equation*}
\begin{split}
  \bm{K}_{(\drug,\sdrug),(\overline{\drug},\overline{\sdrug})}^{\textnormal{MLPK}}
  =& \left(\dkf(\drug, \overline{\drug}) - \dkf(\sdrug,\overline{\drug}) - \dkf(\drug, \overline{\sdrug}) + \dkf(\sdrug, \overline{\sdrug})\right)^2\\
  =& \dkf(\drug, \overline{\drug})^2 + \dkf(\sdrug,\overline{\drug})^2 + \dkf(\drug, \overline{\sdrug})^2 + \dkf(\sdrug, \overline{\sdrug})^2 \\
  &+ 2 \dkf(\drug, \overline{\drug})\dkf(\sdrug, \overline{\sdrug}) + 2 \dkf(\sdrug,\overline{\drug}) \dkf(\drug, \overline{\sdrug}) \\
  &- 2 \dkf(\drug, \overline{\drug})\dkf(\sdrug,\overline{\drug}) - 2\dkf(\drug, \overline{\drug})\dkf(\drug, \overline{\sdrug}) \\
  &- 2\dkf(\sdrug,\overline{\drug}) \dkf(\sdrug, \overline{\sdrug}) - 2 \dkf(\drug, \overline{\sdrug})\dkf(\sdrug, \overline{\sdrug})\\
=& (\bm{Q}(\bm{D}\otimes \bm{D})\bm{Q}^T + \bm{P}\bm{Q}(\bm{D}\otimes \bm{D})\bm{Q}^{T}\\
& + \bm{Q}(\bm{D}\otimes \bm{D})\bm{Q}\bm{P} + \bm{P}\bm{Q}(\bm{D}\otimes \bm{D})\bm{Q}^T\bm{P} + 2(\bm{D} \otimes \bm{D}) \\
&+ 2\bm{P}(\bm{D} \otimes \bm{D}) - 2\bm{Q}(\bm{D} \otimes \bm{D}) - 2\bm{P}\bm{Q}(\bm{D} \otimes \bm{D})-2(\bm{D} \otimes \bm{D})\bm{Q}^T \\ 
&- 2(\bm{D} \otimes \bm{D})\bm{Q}^T \bm{P})_{(\drug,\sdrug),(\overline{\drug},\overline{\sdrug})} \\
  =& [(\bm{I}+\bm{P})(\bm{I}-\bm{Q})(\bm{D}\otimes \bm{D})(\bm{I}-\bm{Q})(\bm{I}+\bm{P})]_{(\drug,\sdrug),(\overline{\drug},\overline{\sdrug})} \\
\end{split}
\end{equation*}


Now, recall that if we have two samples of data, say $X=(\bm{\trdrinds},\bm{\trtainds})$ and $\overline{X}=(\overline{\bm{\trdrinds}},\overline{\bm{\trtainds}})$, and we intend to calculate all kernel evaluations between data in the first sample with the second sample,
the matrix consisting of these kernel evaluations is defined as follows:
\[ \bm{K}= \bm{R}(\overline{\bm{\trdrinds}},\overline{\bm{\trtainds}})\bm{K}^{\textnormal{kernel}}(\bm{D},\bm{T})\bm{R}(\bm{\trdrinds},\bm{\trtainds})^T \]
By setting $\bm{\trdrinds}=\bm{\overline\trdrinds}$ and $\bm{\trtainds}=\bm{\overline\trtainds}$ we may as a special case define the kernel matrix for the training data.


We also have the following rules on permuting either of the indexing matrices with the commutation or the unification operator:
\[  \bm{R}(\bm{\trdrinds},\bm{\trtainds})\bm{P} =  \bm{R}(\bm{\trtainds}, \bm{\trdrinds}) \]
\[  \bm{R}(\bm{\trdrinds},\bm{\trtainds})\bm{Q} =  \bm{R}(\bm{\trdrinds}, \bm{\trdrinds}) \]
\[  \bm{P}^T \bm{R}(\bm{\trdrinds},\bm{\trtainds})^T =  \bm{R}(\bm{\trtainds}, \bm{\trdrinds})^T \]
\[  \bm{Q}^T \bm{R}(\bm{\trdrinds},\bm{\trtainds})^T =  \bm{R}(\bm{\trdrinds}, \bm{\trdrinds})^T \]
To obtain the incomplete data pairwise kernel matrix, we multiply the complete data pairwise kernel matrix $\bm{K}^{\textnormal{kernel}}(\bm{D},\bm{T})$ with the indexing matrix $\bm{R}(\overline{\bm{\trdrinds}},\overline{\bm{\trtainds}})$ and $\bm{R}(\bm{\trdrinds},\bm{\trtainds})^T$ from left and right sides, respectively. The complete data pairwise kernel matrix is a sum of permuted Kronecker product matrices, so these results imply different indexing matrices for each term in the sum. We can then apply GVT to each term separately.
\end{proof}

We have two ways of calculating an identical matrix-vector product given kernel matrices and sample indices $\overline{\bm{\trdrinds}},\overline{\bm{\trtainds}},\bm{\trdrinds},\bm{\trtainds}$, with vectors $\bm{a}\in\mathbb{R}^\trsize,\bm{u}\in\mathbb{R}^t$ :
\begin{enumerate}
    \item Use the standard matrix vector product with the kernel matrix: $\bm{u}\leftarrow \bm{K} \bm{a}$,
    \item Use GVT in Theorem (\ref{mainproposition}):  $\bm{u}\leftarrow\textnormal{vectrick}(\bm{D},\bm{T},\overline{\bm{\trdrinds}},\overline{\bm{\trtainds}},\bm{\trdrinds},\bm{\trtainds},\bm{a})$.
\end{enumerate}
In computing the pairwise kernel matrix $ \bm{K}= \bm{R}(\overline{\bm{\trdrinds}},\overline{\bm{\trtainds}})\bm{K}^{\textnormal{kernel}}(\bm{D},\bm{T})\bm{R}(\bm{\trdrinds},\bm{\trtainds})^T $, only elements in the indexing matrices need to be computed. The computational complexity of implementing approach $1$ directly is $O(\trsize\testsetsize)$. Based on Theorem~\ref{mainproposition} and Corollary~\ref{pairwisetheorem}, the complexity of approach 2 for any of the kernels listed in Table~\ref{table:4a} is  $O(\textnormal{min}(\overline{\tasize}\trsize+\drsize\overline{\trsize},\overline{\drsize}\trsize+\tasize\overline{\trsize}))$. For the training kernel matrix, these complexities can be simplified as $O(\trsize^2$) and $O(\tasize\trsize+\drsize\trsize)$.

\section{Data sets}

\begin{table}[h!]
\centering
\begin{tabular}{ |l|l|l|l|l|l|l|l|l| } 
 \hline
  Data set & Pairs & Drugs & Targ. & Hom. & Dens. & $|\dkm|$ & $|\tkm|$ & $|\bm{K}|$ \\ 
  \hline
  Heterodimer & 5497 & 1526 & 1526 & X & 0.2\% & 3 & & 6\\ 
  Metz & 93 356 & 156 & 1421 & & 42\% & 2 & 2 & 4\\ 
  Merget & 167 995 & 2967 & 226 & & 25\% & 10 & 9 & 4\\ 
  Kernel filling & 8 803 089 & 2967 & 2967 & X & 100\% & 10 & & 6\\ 
 \hline
\end{tabular}
\caption{Data sets used in the experiments.
We report for each data set the number of pairs and unique drugs and targets, and whether the data is homogenous. Density is the fraction of drug target pairs that have known labels.
We denote the number of drug kernels $|\dkm|$, target kernels $|\tkm|$ and pairwise kernels $|\bm{K}|$.}
\label{table:5}
\end{table}

We apply the pairwise kernel learning framework to four biological data sets. As shown in Table \ref{table:5}, the data sets have quite different characteristics. They vary in the number of samples, ratio of drugs to targets, homogeneity, density, and features. While our data sets belong to the same domain, the different prediction tasks provide an useful benchmark on how pairwise kernels perform over different applications.


\subsection{Heterodimer}

Many proteins bind together and form multiprotein structures called protein complexes, which have essential roles in a variety of biological functions. To understand how proteins function, one needs to identify those sets of proteins that form complexes. A significant fraction of known protein complexes are heterodimers, that is, formed by the assembly of only two proteins.  Recent research has taken into account information from measured protein-protein interactions and other possible protein information sources in order to develop new methods for predicting complexes, especially for smaller sizes \citep{ruan2018improving, maruyama2011heterodimeric, ruan2013prediction}.

Labels for a heterodimer data set can be generated from databases of curated protein complexes. We created positive and negative examples following a paper which applied Naive Bayes for supervised learning of heterodimers \citep{maruyama2011heterodimeric}. The labels are based on CYC2008 \citep{pu2008up}, a comprehensive catalogue of 408 manually curated yeast protein complexes, and WI-PHI \citep{kiemer2007wi}, a dataset 49607 (protein,protein)-interactions. A positive (negative) example of a heterodimer is a pair of proteins satisfying the following conditions:
\begin{enumerate}
    \item is (is not) a heterodimeric protein complex in CYC2008
    \item is not (is) a proper subset of any other complex in CYC2008
    \item WI-PHI includes the PPI corresponding to it
\end{enumerate}
This results in a total of 152 positive examples and 5345 negative examples.

Following research that sought to improve heterodimer predictions \citep{ruan2018improving}, we added protein features by considering domain, phylogenetic profile and subcellular localization properties. The idea is that proteins having a similar specification are more likely to form a complex because they are functionally linked. We obtained the domain and subcellular location information from UniProtKB and the phylogenetic profile from KEGG OC \citep{nakaya2012kegg}. The feature map $\phi$ for each of the 1526 proteins is one of three binary vectors (length in parenthesis): 

\begin{enumerate}
    \item $\phi_{\textnormal{dom}}(P_i)_j$: the $j$-th domain occurs in the protein $P_i$ (2554),
    \item $\phi_{\textnormal{phylo}}(P_i)_j$: the $j$-th genome contains the homolog $P_i$ (768),
    \item $\phi_{\textnormal{local}}(P_i)_j$: the $j$-th subcellular localization contains the protein $P_i$ (83).
\end{enumerate}
We computed the protein kernels $\dkm$ using the Tanimoto kernel on these binary feature vectors. Given binary vectors $\bm{v}$ and $\overline{\bm{v}}$  of length $l$, it is defined as the ratio of bits set to 1 in both vs. bits set to 1 in either: $k_{d}(\bm{v},\overline{\bm{v}})=\sum_{i=1}^l\textnormal{min}(v_i,\overline{v}_i)/\sum_{i=1}^l\textnormal{max}(v_i,\overline{v}_i)$.

\subsection{Metz}

Understanding interactions beween chemical compounds and cellular targets is an important research topic in biology.  For example, protein kineases control many aspects of the cell life cycle, and drugs that inhibit specific kineases have been developed to treat several diseases. Large-scale bioactivity assays enable the prediction of interactions across wide panels kinease inhibitors and their potential cellular targets.  In particular, supervised machine learning is a promising approach of predicting interactions since it can use structural similarities among the drug compounds and genomic similarities among target proteins. 

Labels for an interaction data set were based on biochemical selectivity assays for clinically relevant kinease inhibitors by \citet{Metz2011kinome}. The interaction affinity between a ligand molecule (e.g. a drug compound) and a target molecule (e.g. a protein kinease) reflects how tightly the ligand binds to a particular target, quantified using the inhibition constant $K_i$. The smaller the $K_i$ bioactivity, the higher the interaction affinity between the chemical compond and the protein kinase. We binarized the real valued interactions using a relatively stringent threshold of $K_i<28.18\textnormal{nm}$ into 2798 interacting and 90 558 non-interacting pairs.

Following a study that investigated how well machine learning based methods work in different prediction tasks \citep{pahikkala2015dti}, we extracted features for both drugs and targets. Drug features were based on chemical properties, where structural fingerprint similarity was computed as the two dimensional (2D) Tanimoto coefficient based on the structure clustering server at PubChem. Target features were based on genomic data, where sequence similarities were computed using a normalized version of the Smith-Waterman (SW) score. In total, we have 156 drugs and 1421 targets, with a symmetric 156 x 156 (drug, drug)-similarity matrix $X_\drug$ and a symmetric 1421 x 1421 (target,target)-similarity matrix $X_\targ$.  Following the previous study, we used the drug and target similarity matrix rows as feature vectors, computing either a linear kernel $\kf_{\textnormal{Linear}}(\bm{x}_i,\bm{x}_j)=\langle\bm{x}_i,\bm{x}_j\rangle$ or a Gaussian kernel $\kf_{\textnormal{
Gaussian}}(\bm{x}_i,\bm{x}_j)=e^{-\gamma{\|\bm{x}_i - \bm{x}_j\|}^2}$ with bandwidth $\gamma=10^{-5}$ [4]. Assuming that target and drug kernels have the same specification, we then have either linear or Gaussian drug kernels $\dkm$ and target kernels $\tkm$. 

\subsection{Merget}

A study of similar drug bioactivity prediction appears in \citet{cichonska2018learning}, where the task was also to predict the interaction affinity between drug compounds and protein kinease targets. The authors evaluated the pairwise kronecker kernel resulting from 3210 different combinations of 10 drug and 320 target kernels. Many of the pairwise kernels were created by using varying choices of target kernel hyperparameters. This study is interesting for our purposes, because we can use these kernels to evaluate how different pairwise kernels compare on different features.

The labels were created by processing the drug-target interactions in Merget \citep{merget2016profiling} updated with ChEMBL bioactivities \citep{sorgenfrei2018kinome}. The authors used only drugs that had more than 1\% of bioactivities across kineases measured and only kineases with both domain and ATP binding pocket amino acid subsequences at PROSITE \citep{sigrist2012new}. This resulted in 2967 drugs and 226 protein kinases, with a total of 167 995 binding values.

The features were defined directly through multiple kernel functions for both drugs and targets. Drug kernels $\dkm$ were based on Tanimoto kernels using 10 different binary molecular fingerprints obtained with rcdk R package \citep{guha2007chemical}. Given a fixed choice of hyperparameters they had 9 different protein kernels $\tkm$: three Gaussian kernels based on gene ontology (GO) annotations, three kernels based Smith-Waterman (SW) sequence similarities, and three generic string (GS) kernels. Gaussian kernels were based on three GO profiles: molecular function, biological process and cellular components. The SW kernels and GS kernels are both based on three possible amino acid sequences: full kinase sequences, kinase domain subsequences and ATP binding pocket subsequences. These kernels used BLOSUM 50 as amino acid descriptors. These 9 protein kernels were originally expanded into 320 different kernels by varying the choice of hyperparameters.

\subsection{Kernel filling}

In the final experiment, we use the data set in \citet{cichonska2018learning} to define a  novel prediction task that has an even larger data set, in order to use it for scalability experiments. The authors calculated 10 different drug kernels $(\dkm^i)_{i=1...10}$, which can be used both as labels and as features in a kernel filling prediction task. Given $\trsize=2967$ drugs, each drug kernel is a $\dkm^i\in\mathbb{R}^{2967\times2967}$  matrix. If some of the $2967\times2967=8 803 089$ possible entries are missing, they can be predicted using another kernel that has these entries. For a choice of kernels $i\neq j$, denote $Y=\textnormal{vec}(\dkm^i)$ as the label vector and $\dkm^j$  as the drug kernel. The drug kernel is plugged into a pairwise kernel to predict the label vector. 

To create a smaller data set, we can sample an $\trsize\times \trsize$ submatrix from both kernel matrices, and split these entries into $\trsize_{\textnormal{train}}$ training samples and use remaining samples as setting 1 test samples. The entries outside the submatrix are test samples in settings 2, 3, and 4. The original data set is dense and real-valued; each (drug, drug)-pair has a latent feature vector encoded by the second kernel. Because we are predicting kernel encoded similarities of $\trsize$ drugs that belong to the same domain, the data set is also homogeneous.

\section{Experiments}

\begin{figure}[h]
    \centering
    \includegraphics[width=\textwidth]{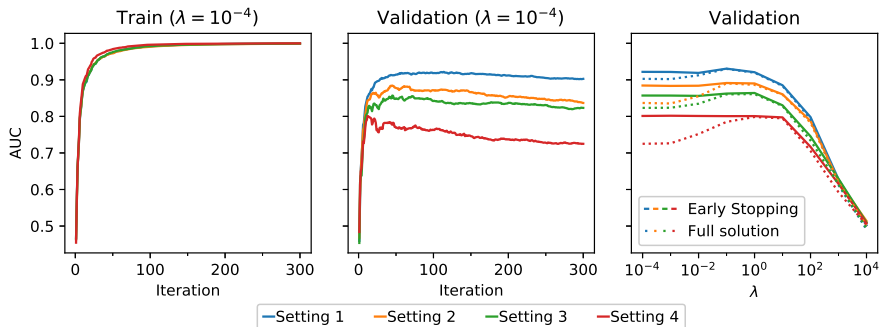}
    \caption{AUC per iteration and the effect of early stopping in the Ki data set.}
    \label{fig:1}
\end{figure}

We implemented ridge regression with the minimum residual optimization method, which is an iterative method for the numerical solution of a system of linear equations. The matrix vector products required within the minimum residual method were computed with either of two algorithms. Given a vector to be multiplied with a pairwise Kronecker kernel matrix, the baseline algorithm uses the explicit kernel matrix and the standard matrix vector product, whereas the fast method uses the GVT algorithm. We used the scipy.sparse.linalg.minres method in the SciPy library. The method CGKronRLS in the RLScore software package includes an user friendly implementation of GVT \citep{pahikkala2016rlscore}, for example. These two methods are identical except for the calculation of the matrix vector products.

Instead of solving the system completely, the minimum residual method can be run up to a given number of iterations. To speed up training, iterations may be stopped before the least squares solution is reached. In practice, a limited number of iterations is often sufficient to reach optimal model performance, where a separate validation set can be used to check whether model performance increases with more iterations. Limiting the number of iterations is also an effective regularization method, known as early-stopping in the literature. A method that includes early stopping therefore has the number of iterations $k$ as a hyperparameter. Regularization can then be performed either by setting the Tikhonov regularization parameter $\lambda$ to a small constant and limiting the number of iterations $k$, or finding an optimal $\lambda$ and stopping iterations when the model has converged. Figure \ref{fig:1} illustrates the effect of early stopping in the Ki data set. The best validation set AUC was reached either by stopping the training early, or by finding the optimal regularization parameter and running the iterations until convergence.

We implemented early stopping ridge regression as follows. The algorithm fits ridge regression $\textnormal{ridge}(Z_{\textnormal{obs}},\dkf,\tkf,\dtkf,\textnormal{setting})$ given a data set $Z_{\textnormal{obs}}$, drug kernel $\dkf$, target kernel $\tkf$, pairwise kernel $\dtkf$, and setting. We use 9-fold cross-validation, according to the setting (see Table~\ref{table:3}), to split the data set into $Z_{\textnormal{train}}$ and $Z_{\textnormal{test}}$ pair during each round. On each round of cross-validation, the training set $Z_{\textnormal{train}}$ is further split into an inner training set $Z_{\textnormal{inner}}$  and a validation set $Z_{\textnormal{validation}}$  according to the setting. The optimal hyperparameter for the number of iterations $k$ is then found by running the minimum residual algorithm on $Z_{\textnormal{inner}}$  until the AUC stops increasing in $Z_{\textnormal{validation}}$ for a given number of iterations.
The number of iterations required and the observed AUC on the validation set is stored. Finally, the model is fit to the full training set $Z_{\textnormal{train}}$ using this many iterations. The resulting model is used to make predictions for the test set $Z_{\textnormal{test}}$, for which the AUC is measured.

\subsection{Heterodimers}

\begin{figure}[h]
    \centering
    \includegraphics[width=1\textwidth]{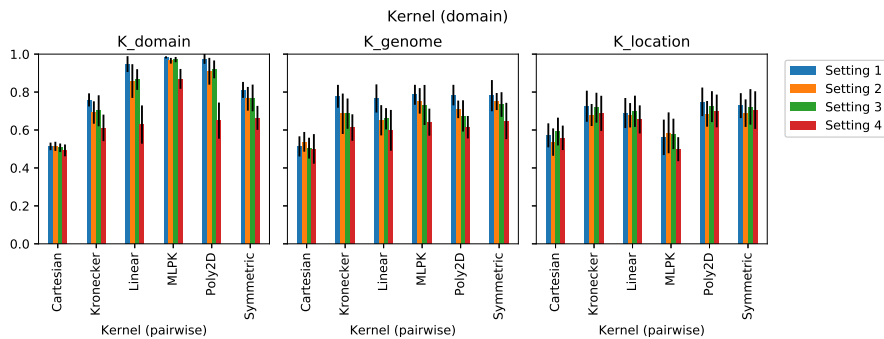}
    \caption{Heterodimers data set: mean and standard deviation of AUCs in test folds for different kernels and settings.}
    \label{fig:heterodimer}
\end{figure}

We tested different pairwise kernels, features and settings in the heterodimers data set.  The experiment included every combination of following choices:
\begin{enumerate}
    \item Drug kernel $\dkf\in\{\dkf^{\textnormal{Domain}},\dkf^{\textnormal{Genome}},\dkf^{\textnormal{Location}}\}$
    \item Pairwise kernel $\dtkf\in\{\dtkf^{\textnormal{Linear}},\dtkf^{\textnormal{Poly2D}},\dtkf^{\textnormal{Kron.}},\dtkf^{\textnormal{Cartesian}},\dtkf^{\textnormal{Symm.}},\dtkf^{\textnormal{MLPK}}\}$
    \item $\textnormal{Setting}\in\{1,2,3,4\}$  splits $Z_{\textnormal{train}}$ into 75\% $Z_{\textnormal{inner}}$ and 25\% $Z_{\textnormal{validation}}$. We fit ridge regression in $Z_{\textnormal{inner}}$ while the AUC in $Z_{\textnormal{validation}}$ is improving.
    \item We then train a $\textnormal{model}\leftarrow\textnormal{ridge}(Z_{\textnormal{train}},\dkf,\tkf,\dtkf,\textnormal{setting})$ with the optimal number of iterations and calculate the AUC corresponding to $\textnormal{Setting}\in\{1,2,3,4\}$ in $Z_{\textnormal{test}}$.

\end{enumerate}

The results in Figure \ref{fig:heterodimer} show that the best pairwise kernel strongly depends on features. For domain features, the MLPK is by far the best pairwise kernel with almost perfect predictions. However, for genome and location features the best kernels are the second degree polynomial and symmetric Kronecker kernels by a notable margin. While the best pairwise kernel depends on the underlying features, using different drug kernel (Min/MinMax/Norm) for the binary feature vectors did not have significant effects, so we report only the Tanimoto, or MinMax, kernel.  The best kernel does not seem to vary by the setting, but the later settings are slightly more challenging. The linear kernel that excludes pairwise interactions, and simply models some proteins having more interactions than others, offers suprisingly good results. However, it seems that in this data set there are also significant pairwise interactions that the other kernels are able to capture.

\subsection{Metz}

\begin{figure}[h]
    \centering
    \includegraphics[width=0.99\textwidth]{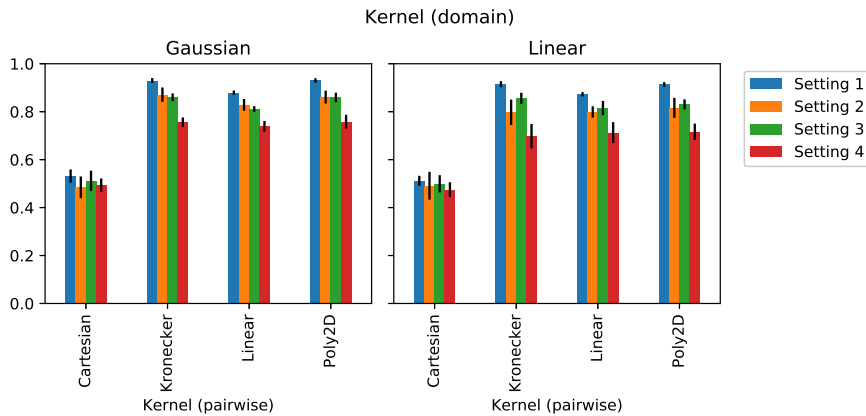}
    \caption{Metz data set: mean and standard deviation of AUCs in test folds for different kernels and settings.}
    \label{fig:metz}
\end{figure}

We tested different pairwise kernels, features and settings in the Metz data set. The experiment included every combination of following choices:
\begin{enumerate}
    \item The drug and target kernels \\$(\dkf,\tkf)\in\{(\dkf^{\textnormal{Linear}},\tkf^{\textnormal{Linear}}),(\dkf^{\textnormal{Gaussian}},\tkf^{\textnormal{Gaussian}})\}$
    \item The pairwise kernel $\dtkf\in\{\dtkf^{\textnormal{Linear}},\dtkf^{\textnormal{Poly2D}},\dtkf^{\textnormal{Kronecker}},\dtkf^{\textnormal{Cartesian}}\}$
    \item
    $\textnormal{Setting}\in\{1,2,3,4\}$  splits $Z_{\textnormal{train}}$ into 75\% $Z_{\textnormal{inner}}$ and 25\% $Z_{\textnormal{validation}}$. We fit ridge regression in $Z_{\textnormal{inner}}$ while the AUC in $Z_{\textnormal{validation}}$ is improving.
    \item We then train a $\textnormal{model}\leftarrow\textnormal{ridge}(Z_{\textnormal{train}},\dkf,\tkf,\dtkf,\textnormal{setting})$ with the optimal number of iterations and calculate the AUC corresponding to $\textnormal{Setting}\in\{1,2,3,4\}$ in $Z_{\textnormal{test}}$.
\end{enumerate}

The results in Figure \ref{fig:metz} show that for both Linear and Gaussian drug kernels, the second degree polynomial and Kronecker pairwise kernels have the best and comparable performance, because they also include pairwise interactions. The linear kernel offers suprisingly good results, not very far from optimal, but there are also some pairwise interactions that contribute to the prediction task. The cartesian kernel is not much better than random on the task. There seem to be some benefits from using the Gaussian instead of the linear drug kernel, which are comparable in magnitude to the benefits from modeling pairwise interactions. Regardless of the drug kernels used as features, over different experiments the relative pairwise kernel performance is the same.

\subsection{Merget}

\begin{figure}[h]
    \centering
    \includegraphics[width=0.99\textwidth]{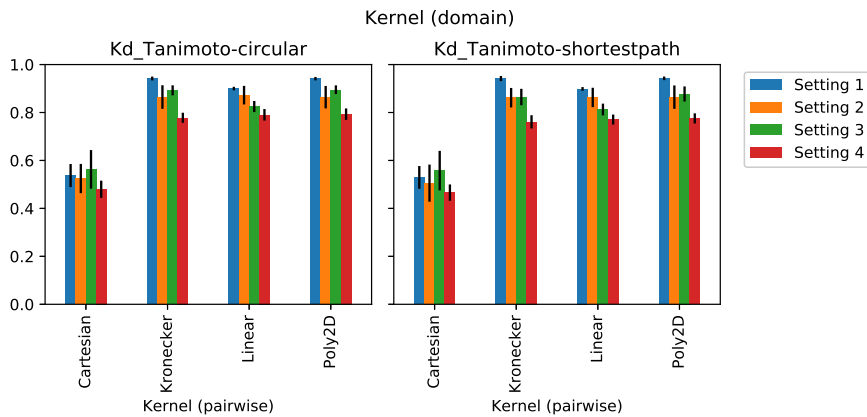}
    \caption{Merget data set: mean and standard deviation of AUCs in test folds for different kernels and settings.}
    \label{fig:3}
\end{figure}

We tested different pairwise kernels, features and settings in the Merget data set. The experiment included every combination of following choices:

\begin{enumerate}
    \item The drug and target kernels 
\begin{equation*}
\begin{split}
    (\dkf,\tkf)\in\{&(\dkf^{\textnormal{sp}},\tkf^{\textnormal{GS−atp−5.4.4}}), (\dkf^{\textnormal{circular}},\tkf^{\textnormal{GS−atp−5.4.4}}),\\
&(\dkf^{\textnormal{kr}},\tkf^{\textnormal{GS−atp−5.4.4}}),
(\dkf^{\textnormal{circular}},\tkf^{\textnormal{GS−kindom−5.4.4}}),\\
&(\dkf^{\textnormal{circular}},\tkf^{\textnormal{GO−bp−71}}),
(\dkf^{\textnormal{circular}},\tkf^{\textnormal{GO−cc−19}}),\\
&(\dkf^{\textnormal{circular}},\tkf^{\textnormal{SW−kindom}}),
(\dkf^{\textnormal{circular}},\tkf^{\textnormal{GS−full−5.3.}})\}
\end{split}
\end{equation*}

    \item The pairwise kernel $\dtkf\in\{\dtkf^{\textnormal{Linear}},\dtkf^{\textnormal{Poly2D}},\dtkf^{\textnormal{Kronecker}},\dtkf^{\textnormal{Cartesian}}\}$
    \item $\textnormal{Setting}\in\{1,2,3,4\}$  splits $Z_{\textnormal{train}}$ into 75\% $Z_{\textnormal{inner}}$ and 25\% $Z_{\textnormal{validation}}$. We fit ridge regression in $Z_{\textnormal{inner}}$ while the AUC in $Z_{\textnormal{validation}}$ is improving.
    \item We then train a $\textnormal{model}\leftarrow\textnormal{ridge}(Z_{\textnormal{train}},\dkf,\tkf,\dtkf,\textnormal{setting})$ with the optimal number of iterations and calculate the AUC corresponding to $\textnormal{Setting}\in\{1,2,3,4\}$ in $Z_{\textnormal{test}}$.
\end{enumerate}

We obtain close to identical results for different (drug kernel, target kernel)-pairs, so we only present the first two pairs. The results in Figure \ref{fig:3} closely mirror the Metz data set. Polynomial and Kronecker kernel are the best with comparable performance in all pairs. Linear kernel has almost as good results, even though some pairwise interactions can be found between the drugs and the targets. Cartesian kernel is not much better than random, with an exception in setting 3. Over all possible drug and target kernel pairs, different features do not seem to have much of an effect on prediction performance or relative order of kernels. This is suprising given that the original study was motivated as a method that enables one to use a large mixture of different kernels to improve prediction performance.

\subsection{Kernel filling}

\begin{figure}[h]
    \centering
    \includegraphics[width=1\textwidth]{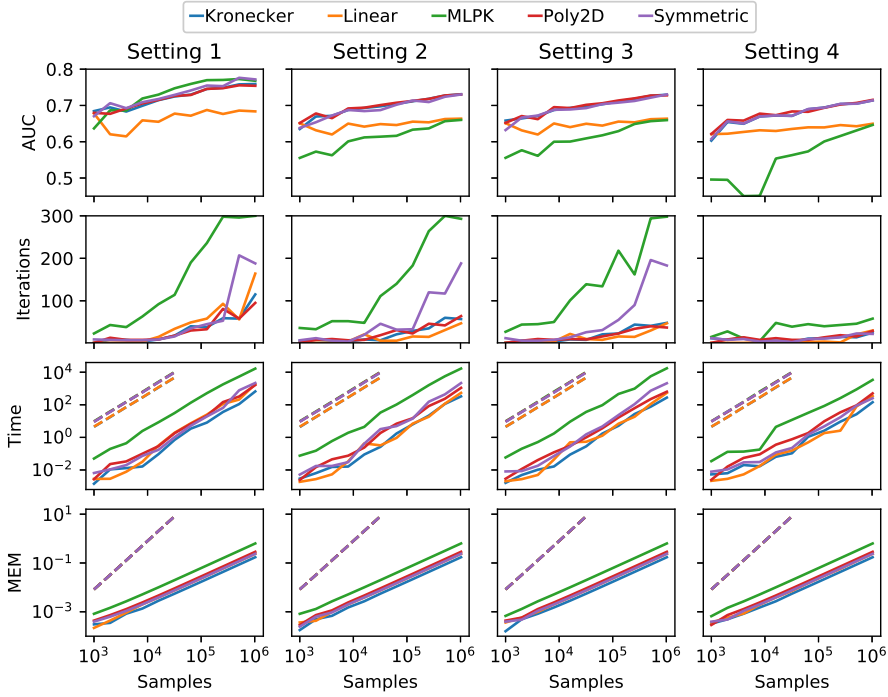}
    \caption{Kernel filling data set: GVT (solid) vs. Baseline (dashed). The AUCs of Kronecker, Poly2D and Symmetric kernels are almost identical and plotted on top of each other. Further, all the baselines have the same memory usage, and their time consumption is close to identical with each other.}
    \label{fig:4}
\end{figure}


We predict the missing labels in a drug kernel matrix $\bm{y}=\textnormal{vec}(\dkm^\textnormal{circular})$  using another drug kernel matrix $\dkm=\dkm^\textnormal{estate}$  as features. Different choices of drug kernels for labels and features result in a drastically different absolute prediction performance. However, not much difference is observed in the relative order of the pairwise kernels. For brevity, we therefore report the experiment on these two kernels, as they offered reasonable but not exceptionally high or low prediction performance.

Because there is so much data available in this task, we used separate test sets. For $N$ samples, the data set $Z_{\textnormal{obs}}$ is split into a  $(Z_{\textnormal{train}},Z_{\textnormal{test}}^{(1)},Z_{\textnormal{test}}^{(2)},Z_{\textnormal{test}}^{(3)},Z_{\textnormal{test}}^{(4)})$-partition by taking a subset of $k$ drugs such that approximately 50\% of the subset results in $Z_{\textnormal{train}}$ with $N$ samples and rest of the subset is $Z_{\textnormal{test}}^{(1)}$, with other drugs defining $Z_{\textnormal{test}}^{(2)},Z_{\textnormal{test}}^{(3)},Z_{\textnormal{test}}^{(4)}$. We then tested different drug and pairwise kernels to test how the number of iterations, CPU time, memory usage and the AUC on the test set is affected by the choice of the pairwise kernel. The experiment included every combination of following choices:
\begin{enumerate}
    \item The pairwise kernel \\ $\dtkf\in\{\dtkf^{\textnormal{Linear}},\dtkf^{\textnormal{Poly2D}},\dtkf^{\textnormal{Kronecker}},\dtkf^{\textnormal{Cartesian}},\dtkf^{\textnormal{Symmetric}},\dtkf^{\textnormal{MLPK}}\}$
    \item The $\textnormal{setting}\in\{1,2,3,4\}$  splits the data set $Z_{\textnormal{train}}$ into 75\% training set $Z_{\textnormal{inner}}$ and 25\% validation set $Z_{\textnormal{validation}}$.
\end{enumerate}
We iteratively fit early stopping ridge regression in $Z_{\textnormal{inner}}$ while the AUC in $Z_{\textnormal{validation}}$ is improving, and then save the optimal number of iterations. We then train the model on $Z_{\textnormal{train}}$ for that many iterations and evaluate the AUC on $Z_{\textnormal{test}}^{(\textnormal{setting})}$.

The number of iterations required to reach an optimal model is shown in Figure \ref{fig:4}. The number of iterations depends on the performance it is possible to achieve in a given setting. More iterations are needed to find a more elaborate model when it is possible to achieve a better prediction performance. Setting 1 requires most iterations, setting 2/3 somewhat less, and setting 4 fewest iterations to reach an optimal solution. Fitting the MLPK and symmetric Kronecker kernel seem to require significantly more iterations than other kernels. Note how the number of iterations is very modest, relative to the total number of samples that is theoretically needed to fully solve the linear system.

The CPU time in seconds and memory usage in GiB are shown in Figure \ref{fig:4}, respectively. The standard method requires significantly more time than the GVT method. At a time when the standard method ran out of memory, the training was taking over an hour whereas GVT completed in a second. The performance of GVT has a small constant term depending on how many summands of Kronecker kernels are required in the pairwise kernel expression. The Kronecker kernel is fastest of these because it has only one term and the MLPK slowest because it has 10 such terms. The naive method requires significantly more memory because it stores the full $O(\trsize^2)$ pairwise kernel matrix whereas GVT only stores the $O(\drsize^2)$ drug and $O(\tasize^2)$ kernel matrices. Here we have $\trsize\approx 0.5q^2$, which implies complexities $O_\textnormal{naive} (\tasize^4)$  vs. $O_\textnormal{GVT} (\tasize^2)$. The naive method experiments were stopped when $N$ required $>16\textnormal{GiB}$ memory, which did not become an issue with GVT for the size of this data set.

Prediction performance comparisons, quantified with the AUC in Figure \ref{fig:4}, are quite complicated because they depend on the setting and the size of the data set. We make the following observations in different settings:
\begin{enumerate}
    \item Setting 1: The MLPK kernel has slightly higher performance for larger data sets $N>10 000$. The Kronecker, second degree polynomial, and symmetric Kronecker kernels are comparable to each other and quite close to the MLPK.  For medium data sets $N\leq 10 000$, they perform better and incorporating prior knowledge via symmetrization may provide a small benefit. The linear kernel is significantly worse except for very small data sets $N\leq 1000$.
    \item Setting 2/3: The settings are equivalent because the domain is homogeneous. The MLPK kernel has worst performance for all data set sizes, and the linear kernel is significantly worse except for the smallest $N\leq 1000$ data sets. The Kronecker, second degree polynomial, and symmetric Kronecker kernels have best and almost identical performance. The general prediction accuracy is somewhat lower because the prediction task has become more difficult
    \item Setting 4: The results are similar to setting 2/3, but the overall prediction accuracy is even slightly lower.
\end{enumerate}

\subsection{Comparison to Nyström approximation with Falkon}
The method proposed in this article allows computing efficiently the exact solution to the regularized risk minimization problem for a family of commonly used pairwise kernels. In the following experiments, we compare the approach to a standard approximation method that speeds up training by using only a random subset of the training data to represent the learned function. Specifically, we compare the proposed method, implemented in RLScore package \citep{pahikkala2016rlscore}, to the Nyström-method based training algorithm implemented in the Falkon package \citep{rudi2007falkon,meanti2020kernel}. The Nyström approximation allows speeding up kernel methods on large data sets, though there is a trade-off in accuracy if the approximation is not sufficient. This introduces an additional hyperparameter: the number of basis vectors $N$ used as an approximation. The method then computes the kernel $\widetilde{\bm{K}}\in \mathbb{R}^{n \times N}$ between data points and basis vectors, and Falkon solves the resulting linear system using a preconditioned conjugate gradient optimizer. Our pairwise kernel implementation inherits the \textsf{falkon.kernels.kernel.Kernel} class from Falkon, and is implemented as c-language extension using Cython to guarantee efficiency. The comparison was implemented on the kernel filling task described in the previous section.

Our experiments in Figure~\ref{fig:falkon_hyperparameters} collaborate theoretic results, which state that increasing the number of basis vectors will result in a higher accuracy when properly regularized. The solution converges to the full solution as the number of basis vectors approaches the number of data points. We also saw that limiting the number of basis vectors effectively regularizes the problem, as the model converges quickly and early stopping results in an identical solution. However, a kernel matrix with 1 024 000 samples and 2048 basis vectors already consumes 16GiB memory so we use this as the best approximation. To align the results with RLScore, we use a regularization parameter $\lambda=1e-5$ and early stopping based on a validation set.

In Figure~\ref{fig:falkon_kronecker}, we perform experiments to compare the Falkon method with $N=32,128,512,2048$ basis vectors against RLScore that computes the full solution using GVT, both with the Kronecker product kernel. The experiments are otherwise identical to the previous experiment with the standard kernel method versus the GVT based method. We see that the quality of the approximation increases with the number of basis vectors, almost reaching the AUC of the RLScore implementation. RLScore has lower runtime and lower memory requirement. We conclude that both methods provide a drastically smaller runtime and memory use compared to the standard method, but are quite comparable to each other in computational requirements.
RLScore provides slightly better AUC with less computational resources, especially in Setting~1.

\begin{figure}[h]
    \centering
    \includegraphics[width=1\textwidth]{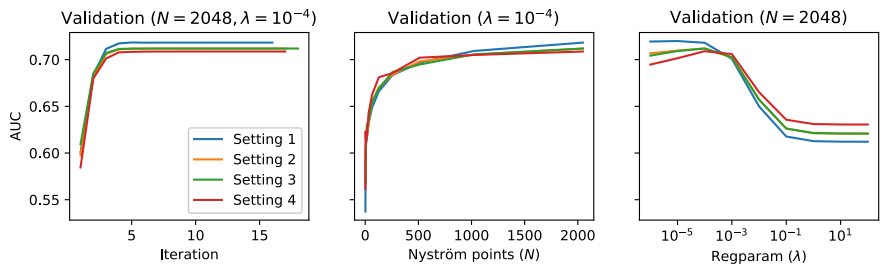}
    \caption{Tuning the hyperparameters of Falkon package with 64 000 data points: the number of basis vectors (middle) and regularization (right). Only a few iterations are required to reach optimal validation AUC (left).}
    \label{fig:falkon_hyperparameters}
\end{figure}

\begin{figure}[h]
    \centering
    \includegraphics[width=\textwidth]{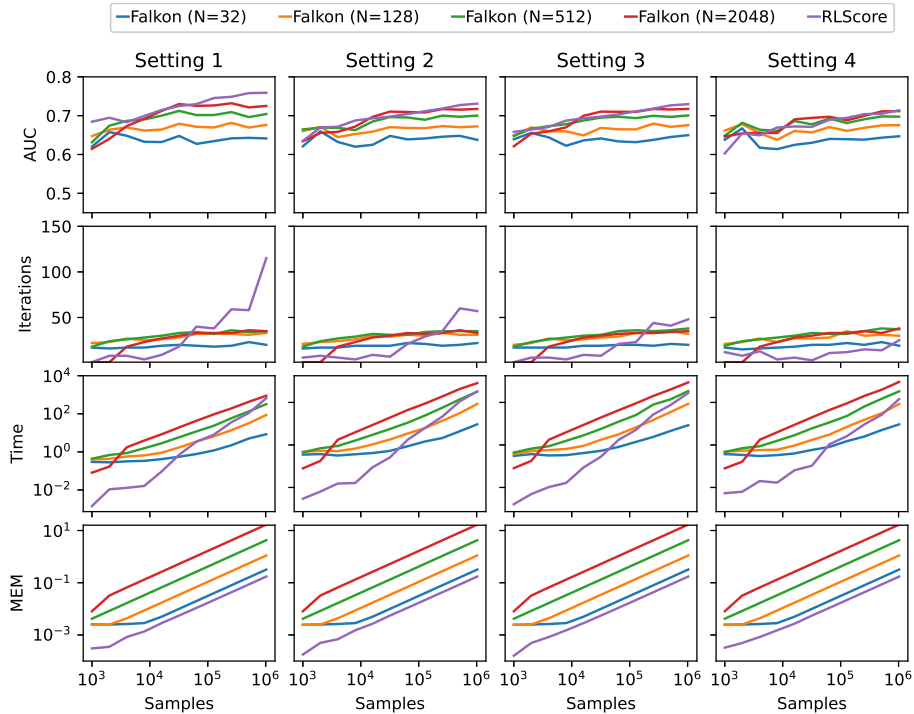}
    \caption{Kronecker kernel: Nyström approximation given a number of basis vectors implemented by the 'Falkon' package vs. full solution implemented with GVT by 'RLScore'.}
    \label{fig:falkon_kronecker}
\end{figure}

\section{Discussion and Conclusion}

In this work we reviewed the most commonly used pairwise kernels and introduced an operator based framework for analysing and implementing the kernels. The framework allows applying the generalized vec-trick algorithm \citep{Airola2017gvt} for speeding up matrix-vector products for the kernels, allowing much faster training and prediction than with explicit computation of the kernel matrix. As a specific use case we considered the ridge regression method, but the approach can be also used for speeding up the (sub)gradient computations for other types of regularized kernel methods, such as kernel logistic regression or support vector machines. Our experiments on drug-target data show that the approach allows scaling to much larger problem sizes than without the computational short cuts, and provides better predictive performance with the same computational resources than Falkon \citep{rudi2007falkon}, a the state-of-the art method for training large-scale kernel machines.
Further, the choice of optimal kernel is seen to be highly dependant on both the problem domain and the type of prediction task considered. 

An interesting observation that can be made from the experimental results is that in many cases the linear pairwise kernel produces results that are competitive with those obtained using the non-linear kernels. This is surprising in the sense that the kernel allows only expressing functions of the form  $f(d,t) =f_{d}(d) +f_{t}(t)$ that score the drugs and targets separately without truly modeling interactions between them. It seems implausible that the non-linearity assumption would not hold in the domain of drug-target interaction prediction, or other similar interaction prediction tasks, since this would imply the existence of "universal drug" that would be the optimal choice for all targets. 
We observed from our experimental results that, with larger sample sizes (see Figure~\ref{fig:4}), the nonlinear kernels were better able to capture the nonlinear components of the underlying signal, and the relative strength of the nonlinear part likely determines how large training sample is needed to capture it. An example showing the extreme cases containing either nonlinear or linear signal components only is given in Figure~\ref{chess}. The nonlinear component also may easily "drown" with high dimensional data, as the number of interaction terms increases quickly as its function.

We make the GVT code publicly available as part of the open source RLScore machine learning library \citep{pahikkala2016rlscore}, allowing other researchers and developers to make use of the described kernel matrix multiplication short cuts. Our work considers the specific case of pairwise data, an open question remains under what conditions similar efficient methods can be derived in general to $\trsize$th order tensorial data, which could be a Kronecker product of more than two kernel matrices. For example, the data may consist of triplets (drug, target, cell line) where each object in the triplet has its own kernel.
\begin{acknowledgements}
  The authors would like to thank Dr. Anna Cichonska for her generous help in providing us with the Merget and Kernel filling data sets.
\end{acknowledgements}

\section*{Declarations}

\noindent \textbf{Funding} This work was supported by Academy of Finland (Grants 311273, 313266, 340140, 340182).

\vspace{\baselineskip}

\noindent \textbf{Conflicts of interest} The authors declare that they have no conflicts of interest.

\vspace{\baselineskip}

\noindent \textbf{Ethics approval} Not Applicable

\vspace{\baselineskip}

\noindent \textbf{Consent to participate} Not Applicable

\vspace{\baselineskip}

\noindent \textbf{Consent for publication} Not Applicable

\vspace{\baselineskip}

\noindent \textbf{Availability of data and material} The data used in the Metz experiments can be downloaded from \url{http://staff.cs.utu.fi/~aatapa/data/DrugTarget/} and Merget and kernel filling data from \url{https://github.com/aalto-ics-kepaco/pairwiseMKL}).

\vspace{\baselineskip}

\noindent \textbf{Code availability} The GVT algorithm for fast training of pairwise kernel methods is made publicly available as part of RLScore package under open access license at \url{https://github.com/aatapa/RLScore}.

\vspace{\baselineskip}

\noindent \textbf{Author's contributions} MV wrote the initial draft of the manuscript, with contributions from AA and TP. All the authors edited together and approved the final version. MV was responsible for data processing and experiments, with experimental design done jointly by all the authors. All the authors contributed to writing code. AA and TP jointly planned the study deriving the joint framework for pairwise kernels, with contributions from MV.

\bibliographystyle{spbasic} 
\bibliography{myBibliography}

\end{document}